\documentclass[11pt, a4paper]{article}

\usepackage[margin=1.19in]{geometry}
\usepackage[utf8]{inputenc} 
\usepackage[T1]{fontenc}    
\usepackage{hyperref}       
\usepackage{url}            
\usepackage{booktabs}       
\usepackage{nicefrac}       
\usepackage{microtype}      
\usepackage{algorithm}
\usepackage[noend]{algorithmic} 
\usepackage{times}
\usepackage{latexsym}
\usepackage{graphicx}

\usepackage{subcaption}
\usepackage[english]{babel} 
\usepackage{amsmath,amsfonts,amsthm} 
\usepackage{enumitem}
\usepackage{wrapfig}
\usepackage{amssymb}
\usepackage{thmtools}
\usepackage{nameref,cleveref,}
\newtheorem{theorem}{Theorem}
\newtheorem{remark}{Remark}
\newtheorem*{theorem*}{Theorem}
\declaretheorem[name=Corollary, sibling=theorem, refname={corollary,corollaries}, Refname={Corollary,Corollaries}]{corollary}
\declaretheorem[name=Lemma, sibling=theorem, refname={lemma,lemmas}, Refname={Lemma,Lemmas}]{lemma}

\declaretheorem[name=Assumption, refname={assumption,assumptions}, Refname={Assumption,Assumptions}]{assumption}
\newcounter{relctr} 
\everydisplay\expandafter{\the\everydisplay\setcounter{relctr}{0}}
 
\newcommand\labelrel[2]{%
  \begingroup
    \refstepcounter{relctr}%
    \stackrel{\textnormal{(\alph{relctr})}}{\mathstrut{#1}}%
    \originallabel{#2}%
  \endgroup
}
\AtBeginDocument{\let\originallabel\label} 
\usepackage{multirow}
\usepackage{pbox}
\usepackage{pgfplots}
\usepackage{fullpage}
\usepackage{subcaption}
\usepackage{authblk}
\input{definitions}

\title{Interaction Hard Thresholding: Consistent Sparse Quadratic Regression in Sub-quadratic Time and Space}
\author[1]{Shuo Yang\thanks{Equal contribution.}}
\author[2]{Yanyao Shen$^*$}
\author[2]{Sujay Sanghavi}
\affil[1]{Department of Computer Science, The University of Texas at Austin}
\affil[2]{Department of ECE,  The University of Texas at Austin}

\begin{document}

\maketitle

\begin{abstract}
Quadratic regression involves modeling the response as a (generalized) linear function of not only the features $\xb^j$, but also of quadratic terms $\xb^{j_1} \xb^{j_2}$. The inclusion of such higher-order ``interaction terms" in regression often provides an easy way to increase accuracy in already-high-dimensional problems. However, this explodes the problem dimension from linear $O(p)$ to quadratic $O(p^2)$, and it is common to look for sparse interactions (typically via heuristics).

In this paper we provide a new algorithm -- Interaction Hard Thresholding (\algname) -- which is the first one to provably accurately solve this problem in {\em sub-quadratic} time and space. It is a variant of Iterative Hard Thresholding; one that uses the special quadratic structure to devise a new way to (approx.) extract the top elements of a $p^2$ size gradient in sub-$p^2$ time and space. 

Our main result is to theoretically prove that, in spite of the many speedup-related approximations, \algname~linearly converges to a consistent estimate under standard high-dimensional sparse recovery assumptions. We also demonstrate its value via synthetic experiments. 

Moreover, we numerically show that \algname~can be extended to higher-order regression problems, and also theoretically analyze  an SVRG variant of \algname.
\end{abstract}

\section{Introduction}

Simple linear regression aims to predict a response $y$ via a (possibly generalized) linear function $\thetab^\top \xb$ of the feature vector $\xb$. {\bf Quadratic regression} aims to predict $y$ as a quadratic function $\xb^\top \Thetab \xb$ of the features $\xb$ 
\begin{eqnarray*}
\text{Linear Model} & \quad & \text{Quadratic Model} \\
y ~ \sim ~ \thetab^\top \xb & & y ~ \sim ~ \xb^\top \Thetab \, \xb
\end{eqnarray*}
The inclusion of such higher-order {\em interaction terms} -- in this case second-order terms of the form $\xb^{j_1} \xb^{j_2}$ -- is common practice, and has been seen to provide much more accurate predictions in several high-dimensional problem settings like
recommendation systems, advertising, social network modeling and computational biology \cite{rendle2010factorization,hao2014interaction,bien2013lasso}. In this paper we consider quadratic regression with an additional (possibly non-linear) link function relating $y$ to $\xb^\top \Thetab \, \xb$.

One problem with explicitly adding quadratic interaction terms is that the dimension of the problem now goes from $p$ to $p^2$. In most cases, the quadratic problem is high-dimensional and will likely overfit the data; correspondingly, it is common to implicitly / explicitly impose low-dimensional structure on the $\Thetab$ -- with sparsity of $\Thetab$ being a natural choice. A concrete example for sparse interaction would be the genome-wide association study, where for a given phenotype, the associated genetic variants are usually a sparse subset of all possible variants. Those genes usually interact with each other and leads to the given phenotype \cite{Li2015ModelingGI}.

The {\bf naive approach} to solving this problem involves recasting this as a big linear model that is now in $p^2$ dimensions, with the corresponding $p^2$ features being all pairs of the form $\xb^{j_1} \xb^{j_2}$. However, this approach takes $\Omega(p^2)$ time and space, since sparse linear regression cannot be done in time and space smaller than its dimension -- which in this case is $p^2$ -- even in cases where statistical properties like restricted strong convexity / incoherence etc. hold. Fundamentally, the problem lies in the fact that one needs to compute a gradient of the loss, and this is an $\Omega(p^2)$ operation.

{\bf Our motivation:} Can we learn a sparse quadratic model with time and space complexity that is sub-quadratic? In particular, suppose we have data which is well modeled by a $\Thetab^*$ that is $K$-sparse, with $K$ being $\Ocal(p^\gamma)$ and $\gamma < 1$. Statistically, this can be possibly recovered from $\Ocal(K\log p)$ samples, each of which is $p$-dimensional. Thus we have a setting where the input is sub-quadratic with size $\Ocal(Kp\log p)$, and the final output is sub-quadratic with size $\Ocal(K)$. Our aim is to have an algorithm whose time and space complexity is also sub-quadratic for this case.

{\bf In this paper}, we develop a new algorithm which has this desired sub-quadratic complexity, and subsequently theoretically establish that it consistently recovers a sparse $\Thetab^*$. We briefly overview our setting and results below.

\subsection{Main Contributions}

Given $n$ samples $\{ (\xb_i, y_i) \}_{i=1}^{n}$, we are interested in minimizing the following loss function corresponding to a quadratic model:
\begin{align} \label{eqt:quadra}
     \mbox{\textbf{ (Quadratic Structure) }}  \quad  \min_{\Thetab : \|\Thetab \|_0 \leq K} ~ \frac{1}{n} \, \sum_{i=0}^{n-1} \,  f\left( \xb_i^\top\Thetab\xb_i, y_i \right) \quad := ~ F_n\left(\Thetab \right ) \quad  
\end{align}

We develop a {\bf new algorithm} -- Interaction Hard Thresholding (\algname), outlined in {\bf \Cref{alg:qht}} -- for this problem, and provide a {\bf rigorous proof of consistency} for it under the standard settings (Restricted strong convexity and smoothness of the loss) for which consistency is established for sparse recovery problems. At a high level, it is based on the following {\bf key ideas:} 

\begin{itemize}[leftmargin=*]
    \item[{\bf (1)}] Because of the special quadratic structure, we show that the top $2k$ entries of the gradient can be found in sub-quadratic time and space, using ideas from \textcolor{\bcolor}{hashing and coding}. The subroutine in \Cref{alg:sketch} for doing this is { based on the idea of \cite{pagh2013compressed}} and \textbf{\Cref{thm:top-recovery}} characterizes its performance and approximation guarantee. 
    \item[{\bf (2)}] We note a simple but key fact: in (stochastic) iterative hard thresholding, the new $k$-sparse $\Thetab_{t+1}$ that is produced has its support inside the union of two sets of size $k$ and $2k$: the support of the previous $\Thetab_t$, and the top-$2k$ elements of the gradient. 
    \item[{\bf (3)}] While we do not find the precise top-$2k$ elements of the gradient, we do find an approximation. Using a new theoretical analysis, we show that this approximate-top-$2k$ is still sufficient to  establish linear convergence to a consistent solution. This is our main result, described in \textbf{\Cref{thm:main2}}. 
    \item[{\bf (4)}] As an extension, we show that our algorithm also works with popular SGD variants like SVRG  (\textbf{\Cref{alg:qht-svrg} in \Cref{appendix:SVRG}}), with provable linear convergence and consistency \textbf{ in \Cref{sec:extension}}. We also demonstrate the extension of our algorithm to estimate higher order interaction terms with a numerical experiment \textbf{in \Cref{sec:experiments} }.
\end{itemize}

\textbf{Notation} 
We use $[n] $ to represent the set $ \{0,\cdots, n-1\}$. 
We use $f_{\Bcal}\left(\Thetab\right)$ to denote the average loss on  batch $\Bcal$, where $\Bcal $ is a subset of $[n]$ with batch size $m$.    
We define $\minnerprod{\Ab,\Bb} = \tr{(\Ab^\top \Bb)}$, and $\supp(\Ab)$ to be the index set of $\Ab$ with non-zero entries. 
We let $\Pcal_S$ to be the projection operator onto the index set $S$. We use standard Big-$\mathcal{O}$ notation for time/space complexity analysis, and Big-$\widetilde{\Ocal}$ notation which ignores log factors.

\section{Related Work}

\textbf{Learning with high-order interactions} Regression with interaction terms has been studied in the statistics community. 
However, many existing results consider under the assumption of strong/weak hierarchical (SH/WH) structure: the coefficient of the interaction term $\xb^{j_1}\xb^{j_2}$  is non-zero only when both coefficients of $\xb^{j_1}$ and $\xb^{j_2}$  are  (or at least one of them is) non-zero. 
Greedy heuristics \cite{wu2010screen,hao2014interaction} and regularization based methods \cite{choi2010variable,bien2013lasso,lim2015learning,she2018group,hao2018model} are proposed accordingly. 
However, they could potentially miss important signals that only contains the effect of interactions. 
Furthermore, several of these methods also suffer from scaling problems due to the quadratic scaling of the parameter size. 
There are also results considering the more general tensor regression,  see, e.g., \cite{yu2016learning,hao2018sparse}, among many others. However, neither do these results focus on solutions with efficient memory usage and time complexity, which may become a potential issue when the dimension scales up. 
From a combinatorial perspective, 
\cite{mansour1995randomized,kocaoglu2014sparse}  learns sparse polynomial in Boolean domain using quite different approaches. 

\textbf{Sparse recovery, IHT and stochastic-IHT}
IHT \cite{blumensath2009iterative} is one type of sparse recovery algorithms that is proved to be effective for  M-estimation  \cite{jain2014iterative} under the regular RSC/RSM assumptions.  
\cite{nguyen2017linear} proposes and analyzes a stochastic  version of IHT.  \cite{li2016nonconvex,shen2017tight} further consider  variance reduced acceleration algorithm under this high dimensional setting. 
Notice that IHT, if used for our quadratic problem, still suffers from quadratic space, similar to other techniques, e.g., the Lasso, basis pursuit, least angle regression \cite{tibshirani1996regression,chen2001atomic,efron2004least}. 
On the other hand,  \cite{murata2018sample} recently considers a variant of IHT, where for each sample, only a random subset of features is observed. This makes each update cheap, but their sample size has linear dependence on the ambient dimension, which is again quadratic. 
Apart from that, \cite{nguyen2017linear,liu2018high} also show that IHT can potentially tolerate a small amount of error per iteration .

\textbf{Maximum inner product search}
One key technique of our method is extracting the top elements (by absolute value) of gradient matrix, which can be expressed as the inner product of two matrices. 
This can be formulated as finding Maximum Inner Product (MIP) from two sets of vectors. 
In practice, algorithms specifically designed for MIP are proposed based on locality sensitive hashing \cite{shrivastava2014asymmetric}, and many other greedy type algorithms \cite{ballard2015diamond, yu2017greedy}. But they either can't fit into the regression setting, or suffers from quadratic complexity.
{In theory,} MIP is treated as a fundamental problem in the recent development of complexity theory \cite{abboud2017distributed,williams2018difference}. 
\cite{abboud2017distributed,chen2018hardness} shows the hardness of MIP, even for Boolean vectors input. While in general hard, there are data dependent approximation guarantees, using the compressed matrix multiplication method \cite{pagh2013compressed}, which inspired our work.

\textbf{Others}
The quadratic problem we study also share similarities with several other problem settings, including factorization machine
\cite{rendle2010factorization} and kernel learning \cite{shawe2004kernel,rahimi2008random}. Different from factorization machine, we do not require the input data to be sparse. While the factorization machine tries to learn a low rank representation,  we are interested in learning a sparse representation. Compared to kernel learning, especially the quadratic / polynomial kernels, our task is to do feature selection and identify the correct interactions. 

\section{Interaction Hard Thresholding}
\label{sec:alg}

We now describe the main ideas motivating our approach, and then formally describe the algorithm. 

\begin{algorithm}[t]
\centering 
\begin{algorithmic}[1]
\STATE \textbf{Input:} Dataset $\{\xb_i, y_i\}_{i=1}^{n}$, dimension $p$ 
\STATE \textbf{Parameters:}  Step size $\eta$, estimation sparsity $k$, batch size $m$, round number $T$\\
\STATE \textbf{Output:}  The parameter estimation $\widehat{\Thetab}$
\STATE Initialize $\Thetab^0$ as a $p\times p$ zero matrix. 
\FOR{$t=0$ \textbf{to} $T-1$}
\STATE Draw a subset of indices $\Bcal_t$ from $[n]$ randomly. 
\STATE Calculate the residual $u_i=u(\Thetab^t, \xb_i, y_i)$ based on \cref{eqt:quad-grad}, for every $ i \in \Bcal_t$.
\STATE Set $\Ab_t\in\RR^{p\times m}$ , where each column of $\Ab_t$ is $u_i\xb_i$, $i\in\Bcal_t$.
\STATE Set $\Bb_t\in\RR^{p\times m}$, where each column of $\Bb_t$ is $\xb_i$, $i\in\Bcal_t$. (where $\frac{\Ab_t\Bb_t^\top}{m}$ gives the gradient)
\STATE Compute $\widetilde{S}_t = \textsc{\algsubroute}(\Ab_t, \Bb_t, 2k)$. \textit{\hfill ----/* approximate top elements extraction */----\textcolor{white}{$\frac{X^2}{U^\top}$}}
\STATE Set $S_t =\widetilde{S}_t\cup \supp(\Thetab^t)$. \textit{\hfill ----/* inaccurate hard thresholding update */----\textcolor{white}{$\frac{X^2}{U^\top}$}}
\STATE Compute $\Pcal_{S_t}(\Gb^t)\leftarrow$ the gradient value $\Gb^t = \frac{1}{m} \sum_{i\in\Bcal^t} u_i \xb_i\xb_i^\top$ only calculated on $S_t$.
\STATE Update $\Thetab^{t+1} = \Hcal_{k}\left( \Thetab^{t} - \eta \Pcal_{S_t}(\Gb^t)  \right)$.
\ENDFOR 
\STATE \textbf{Return:} $\widehat{\Thetab} = \Thetab^T$
\end{algorithmic}
\caption{\textsc{Interaction Hard Thresholding (\algname)} }
\label{alg:qht}
\end{algorithm}

\textbf{Naively recasting as a linear model has $p^2$ time and space complexity:} As a first step to our method, let us see what happens with the simplest approach. Specifically, as noted before,  problem (\ref{eqt:quadra}) can be recast as one of finding a sparse (generalized) linear model in the $p^2$ size variable $\Thetab$:
\begin{eqnarray*}
\textbf{(Recasting as linear model)} & \quad  \quad & \min_{\Thetab : \|\Thetab \|_0 \leq K} ~ \frac{1}{n} \, \sum_{i=0}^{n-1} \,  f\left(\, \langle \textbf{X}_i , \Thetab \rangle , y_i \, \right)   
\end{eqnarray*}
where matrix $\textbf{X}_i := \xb_i \xb_i^\top$.  Iterative hard thresholding (IHT) \cite{blumensath2009iterative} is a state-of-the-art method (both in terms of speed and statistical accuracy) for such sparse (generalized) linear problems. This involves the following update rule
\begin{eqnarray*}
\textbf{(standard IHT)} & \quad \quad & \Thetab^{t+1} ~ = ~ \Hcal_k \, \left( \, \Thetab^t \, - \, \eta \, \nabla F_n(\Thetab^t) \, \right ) 
\end{eqnarray*}
where $F_n(\cdot)$ is the average loss defined in (\ref{eqt:quadra}), and $\Hcal_k(\cdot)$ is the hard-thresholding operator that chooses the largest $k$ elements (in terms of absolute value) of the matrix given to it, and sets the rest to 0. \textcolor{\bcolor}{Here, $k$ is the  estimation sparsity parameter.} In this update equation, the current iterate $\Thetab^t$ has $k$ non-zero elements and so can be stored efficiently. But the gradient $\nabla F_n(\Thetab^t)$ is $p^2$ dimensional; this causes IHT to have $\Omega(p^2)$ complexity. This issue remains even if the gradient is replaced by a stochastic gradient that uses fewer samples, since even in a stochastic gradient the number of variables remains $p^2$.

\textbf{A key observation:} We only need to know the top-$2k$ elements of this gradient $\nabla F_n(\Thetab^t)$, because of the following simple fact: if $\Ab$ is a $k$-sparse matrix, and $\Bb$ is any matrix, then
\[
\supp(\Hcal_k(\Ab+\Bb)) \subset \supp(\Ab) \cup \supp(\Hcal_{2k}(\Bb)).
\]
That is, the support of the top $k$ elements of the sum $\Ab+\Bb$ is inside the union of the support of $\Ab$, and the top-$2k$ elements of $\Bb$. The size of this union set is at most $3k$.

Thus, in the context of standard IHT, we do not really need to know the full (stochastic) gradient $\nabla F_n(\Thetab^t)$; instead we only need to know (a) the values and locations of its top-$2k$ elements, and (b) evaluate at most $k$ extra elements of it -- those corresponding to the support of the current $\Thetab^t$.

\begin{algorithm}[t]
\centering 
\begin{algorithmic}[1]
\STATE \textbf{Input:} Matrix $\Ab$, matrix $\Bb$, top selection size $k$
\STATE \textbf{Parameters:}  Output set size upper bound $b$, repetition number $d$, significant level $\Delta$
\STATE \textbf{Expected Output: } Set $\Lambda$:  the top-$k$ elements in $\Ab\Bb^\top$ with absolute value  greater than $\Delta$ 
\STATE \textbf{Output:} Set $\widetilde\Lambda$ of indices, with size at most $b$ (approximately contains $\Lambda$)
\\\hrulefill 
\STATE \textbf{Short Description: }This algorithm is adopted directly from \cite{pagh2013compressed}. It follows from the matrix compressed product via FFT (see section 2.2 of \cite{pagh2013compressed}) and sub-linear result extraction by error-correcting code (see section 4 of \cite{pagh2013compressed}), which drastically reduces the complexity. The whole process is repeated for $d$ times to boost the success probability. The notation here matches \cite{pagh2013compressed} exactly, except that we use $p$ for dimension while $n$ is used in \cite{pagh2013compressed} instead.
\STATE Intuitively, the algorithm will put all the elements of $\Ab\Bb^\top$ into b different "basket"s, with each of the elements assigned a positive or negative sign. It then selects the "basket" whose magnitude is greater than $\Delta$. Further, one large element is recovered from each of the selected baskets. 
\end{algorithmic}
\caption{\textsc{Approximated Top Elements Extraction (\algsubroute)} }
\label{alg:sketch}
\end{algorithm}

The \textbf{key idea} of our method is to exploit the special structure of the quadratic model to find the top-$2k$ elements of the batch gradient $\nabla f_{\Bcal}$ in sub-quadratic time. Specifically, $\nabla f_{\Bcal}$
has the following form:
\begin{align}
     \nabla f_{\Bcal}(\Thetab) \triangleq \frac{1}{m} \sum_{i \in \Bcal} \nabla f\left(\xb_i^\top \Thetab\xb_i, y_i\right) = \frac{1}{m} \sum_{i\in \Bcal} u( \Thetab, \xb_i, y_i) \xb_i\xb_i^\top, \label{eqt:quad-grad} 
\end{align}
where $u( \Thetab, \xb_i, y_i )$ is a scalar related to the residual and the derivative of link function
, and $\Bcal$ represents the mini-batch where $\Bcal \subset \left[n\right], |\Bcal| = m$.
This allows us to approximately find the top-$2k$ elements of the $p^2$-dimensional stochastic gradient in  $\widetilde{\Ocal}(k(p+k))$ time and space, which is sub-quadratic when $k$ is $\Ocal(p^\gamma)$ for $\gamma <1$. 

Our algorithm is formally described in \Cref{alg:qht}.  We use Approximate Top Elements Extraction (\algsubroute) to approximately find the top-$2k$ elements of the gradient, which is briefly summarized in \Cref{alg:sketch}, based on the idea of Pagh \cite{pagh2013compressed}. The full algorithm is re-organized and provided in \Cref{sec:atee-formal} for completeness.   
Our method, Interaction Hard Thresholding (\algname) builds on IHT, but needs a substantially new analysis for proof of consistency. The subsequent section goes into the details of its analysis.

\section{Theoretical Guarantees}\label{sec:guarantees}

In this section, we establish the consistency of Interaction Hard Thresholding, in the standard setting where sparse recovery is established.

Specifically, we establish  convergence results under deterministic assumptions on the data and function, including restricted strong convexity (RSC) and smoothness (RSM). Then, we analyze the sample complexity when features are generated from sub-gaussian distribution in the quadratic regression setting, in order to have well-controlled RSC and RSM parameters. The analysis of required sample complexity yields an overall complexity that is sub-quadratic in time and space. 

\subsection{Preliminaries}

We first describe the standard deterministic setting in which sparse recovery is typically analyzed. Specifically, the samples $(\xb_i,y_i)$ are fixed and known. Our first assumption defines how our intended recovery target $\Thetab^\star$ relates to the resulting loss function $F_n(\cdot)$.

\begin{assumption}[Standard identifiability assumption]\label{as:Noise}
 There exists a $\Thetab^\star$ which is $K$-sparse such that the following holds: given any batch $\Bcal \subset [n]$ of $m$ samples, the norm of batch gradient at $\Thetab^\star$ is bounded by  constant $G$. That is,
$\fnorm{\nabla f_{\Bcal}(\Thetab^\star)}\le G$, and $\left\|{\Thetab^\star}\right\|_{\infty} \le \omega$.
\end{assumption}

In words, this says the the gradient at $\Thetab^\star$ is small. In a noiseless setting where data is generated from $\Thetab^\star$, e.g. when $y_i = \xb_i^\top \Thetab^\star \xb_i$, this gradient is 0; i.e. the above is satisfied with $G=0$, and $\Thetab^\star$ would be the exact sparse optimum of $F_n(\cdot)$. The above assumption generalizes this notion to noisy and non-linear cases, relating our recovery target  $\Thetab^\star$ to the loss function. This is a standard setup assumption in sparse recovery.

Now that we have specified what $\Thetab^\star$ is and why it is special, we specify the properties the loss function needs to satisfy. These are again standard in the sparse recovery literature  \cite{nguyen2017linear,shen2017tight,li2016nonconvex}.  

\begin{assumption}[Standard landscape properties of the loss]\label{as:summary}
For any pair $\Thetab_1, \Thetab_2$ and $s\leq p^2$ such that $|\supp(\Thetab_1-\Thetab_2)| \le s$
\begin{itemize}[leftmargin=*]
    \item The overall loss $F_n$ satisfies $\alpha_s$-Restricted Strong Convexity (RSC): 
    \begin{align*}
        F_n(\Thetab_1) - F_n(\Thetab_2) \ge \minnerprod{  \Thetab_1 - \Thetab_2, \nabla_{\Thetab} F_n(\Thetab_2) } + \frac{\alpha_s}{2} \left\| \Thetab_1 - \Thetab_2 \right\|_F^2
    \end{align*}
    \item The mini-batch loss $f_{\Bcal}$ satisfies $L_s$-Restricted Strong Smoothness (RSM): 
    \begin{align*}
        \fnorm{ \nabla f_{\Bcal}(\Thetab_1) - \nabla f_{\Bcal}(\Thetab_2) } \le L_s \fnorm{ \Thetab_1 - \Thetab_2 },~\forall \Bcal \subset \left[n\right],~|\Bcal|=m
    \end{align*}
    \item  $f_{\Bcal}$ satisfies Restricted Convexity (RC) (but not strong):
    \begin{align*}
        f_{\Bcal}(\Thetab_1) - f_{\Bcal}(\Thetab_2) - \minnerprod{  \nabla f_{\Bcal}(\Thetab_2), \Thetab_1 - \Thetab_2 } \ge 0,~\forall \Bcal \subset \left[n\right],~|\Bcal|=m,~s=3k+K
    \end{align*}
\end{itemize}
\end{assumption}

{\bf Note:} While our assumptions are standard, our result does not follow immediately from existing analyses -- because we cannot find the exact top elements of the gradient. We need to do a new analysis to show that even with our approximate top element extraction, linear convergence to $\Thetab^\star$ still holds.

\subsection{Main Results }

Here we proceed to establish the sub-quadratic complexity and consistency of \algname~for parameter estimation. \Cref{thm:top-recovery} presents the analysis of \algsubroute. It provides the computation complexity  analysis, as well as the statistical guarantee of support recovery. Based on this,  we show the per round convergence property of \Cref{alg:qht} in \Cref{thm:iht}. We then establish our main statistical result, the linear convergence of \Cref{alg:qht} in \Cref{thm:main2}.

Next, we discuss the batch size that guarantees support recovery in \Cref{thm:batchsize}, focusing on the quadratic regression setting, i.e. the model is linear in both interaction terms and linear terms. Combining all the established results, the sub-quadratic complexity is established in \Cref{coro:sub-quadratic}. All the proofs in this subsection can be found in \Cref{prf:guarantees}.

\textbf{Analysis of \algsubroute} Consider \algsubroute~with parameters set to be $b,d,\Delta$. Recall this means that \algsubroute~ returns an index set  $(\widetilde{\Lambda})$ of size at most $b$, which is expected to contain the desired index set ($\Lambda$). Note that the desired index set ($\Lambda$) is composed by the top-$2k$ elements of gradient $\nabla f_{\Bcal}(\Thetab)$ whose absolute value is  greater than $\Delta$. Suppose now the current estimate is $\Thetab$, and $\Bcal$ is the batch. The following theorem establishes when this output set $(\widetilde\Lambda)$ captures the top elements of the gradient.

\begin{theorem}[Recovering top-$2k$ elements of the gradient, modified from \cite{pagh2013compressed}] \label{thm:top-recovery}
With the setting above, if we choose $b,d,\Delta$ so that $b\Delta^2 \ge 432\fnorm{\nabla f_{\Bcal}(\Thetab)}^2$ and $d \ge 48 \log 2 ck $, then the index set  $(\widetilde\Lambda)$ returned by \algsubroute~ contains the desired index set ($\Lambda$) with probability at least $1-{1}/{c}$. 

Also in this case the time complexity of \algsubroute~ is $\widetilde\Ocal\left(m(p + b)\right)$, and space complexity is $\widetilde\Ocal\left(m(p + b)\right)$.
\end{theorem}

\Cref{thm:top-recovery} requires that parameter $b, \Delta$ are set to satisfy $b \Delta^2 \ge 432\fnorm{\nabla f_{\Bcal}(\Thetab)}^2$. Note that $\Delta$ controls the minimum magnitude of top-$k$ element we can found. To avoid getting trivial extraction result, we need to set $\Delta$ as a constant that doesn't scale with $p$. In order to control the scale of $\Delta$ and $b$, to get consistent estimation and to achieve sub-quadratic complexity, we need to upper bound $\fnorm{\nabla f_{\Bcal}(\Thetab)}^2$. This is the \textit{compressibility estimation} problem that was left open in \cite{pagh2013compressed}. 
In our case, the batch gradient norm can be controlled by the RSM property. More formally, we have
\begin{lemma}[Frobenius norm bound of gradient]\label{lemma:Frobenius}
The Frobenius norm of batch gradient at arbitrary $k$-sparse $\Thetab$, with $\norm{\Thetab}_\infty \le \omega$, can be bounded as $\fnorm{\nabla f_{\Bcal}(\Thetab)} \le 2L_{2k}\sqrt{k}\omega + G$, where $G$ is the uniform bound on $\fnorm{\nabla f_{\Bcal}(\Thetab^\star)}$ over all batches $\Bcal$ and $\omega$ bounds $\norm{\Thetab^\star}_\infty$ (see \Cref{as:Noise}).
\end{lemma}
\Cref{lemma:Frobenius} directly implies that \Cref{thm:top-recovery} could allow $b$ scale linearly with $k$ while keep $\Delta$ as a constant\footnote{For now, we assume $L_{2k}$ to be a constant independent of $p, k$. We will discuss this in \Cref{thm:batchsize}.}. This is the key ingredient to achieve sub-quadratic complexity and consistent estimation.
We postpone the discussion for complexity to later paragraph, and proceed to finish the statistical analysis of gradient descent.

 {\bf Convergence of \algname:} Consider \algname~with parameter set to be $\eta, k$. 
 For the purpose of analysis, we keep the definition of $\Lambda$ and $\widetilde\Lambda$ from the analysis of \algsubroute~and further define $k_\Delta$ to be the number of top-$2k$ elements whose magnitude is below $\Delta$. Recall that $K$ is the sparsity of $\Thetab^\star$, define $\nu = 1+\left(\rho + \sqrt{(4+\rho)\rho}\right)/2, \rho={K}/{k}$, where $\nu$ measures the error induced by exact IHT (see \Cref{lemma:tightbound} for detail).
Denote $B_t = \left\{\Bcal_0, \Bcal_1,...,\Bcal_t\right\}$. We have
\begin{theorem}[Per-round convergence of  \algname]\label{thm:iht} Following the above notations,   the per-round convergence of  \Cref{alg:qht} satisfies the following:
\begin{itemize}[leftmargin=*]
    \item If \algsubroute~ succeeds, i.e., $\Lambda \subseteq \widetilde\Lambda$, then
    \begin{align*}
		 \EE_{B_t}\left[ \fnorm{\Thetab^t - \Thetab^\star}^2 \right] \le \kappa_1\EE_{B_{t-1}} \left[\fnorm{\Thetab^{t-1} - \Thetab^\star}^2\right] + \sigma_{GD}^2 + \sigma_{\Delta|GD}^2,
\end{align*}
where $\kappa_1 = \nu \left(1 - 2\eta \alpha_{2k} +  2\eta^2 L^2_{2k} \right)$, $\sigma_{\Delta|GD}^2 = 4\sqrt{k_\Delta}\eta \sqrt{k}\omega\Delta + 2k_\Delta\eta^2\Delta^2$, and
\begin{align*}
\sigma_{GD}^2=\max_{|\Omega|\le2k+K}\left[4\nu \eta \sqrt{k}\omega\fnorm{\Pcal_{\Omega}\left(\nabla F\left(\Thetab^\star\right)\right) }+ 2\nu \eta^2\EE_{\Bcal_t}\left[\fnorm{\Pcal_{\Omega} \left(\nabla f_{\Bcal_t} \left(\Thetab^\star\right)\right)}^2\right]\right].
\end{align*}

    \item If \algsubroute~ fails, i.e., $\Lambda \not\subset \widetilde\Lambda$, then, 
    \begin{align*}
		 \EE_{B_t}\left[ \fnorm{\Thetab^t - \Thetab^\star}^2 \right] \le \kappa_2\EE_{B_{t-1}}\left[\fnorm{\Thetab^{t-1} - \Thetab^\star}^2\right] + \sigma_{GD}^2 + \sigma_{Fail|GD}^2,
\end{align*}
where
$\kappa_2 = \kappa_1 + 2\nu \eta L_{2k},~~\sigma_{Fail|GD}^2 = \max_{|\Omega|\le 2k+K}\left[4\nu \eta\sqrt{k}\omega\EE_{\Bcal_t}\left[\fnorm{\Pcal_{\Omega}\left(\nabla f_{\Bcal_t}\left(\Thetab^\star\right)\right) }\right]\right]
$.

\end{itemize}
\end{theorem}

\begin{remark}\label{rm:SGDsigma}
	It is worth noting that $\sigma_{GD}, \sigma_{Fail|GD}$ are both statistical errors, which in the noiseless case are $0$. In the case that the magnitude of top-$2k$ elements in the gradient are all greater than $\Delta$, we have $k_\Delta = 0$, which implies $\sigma_{\Delta|GD}=0$. In this case \algsubroute's approximation doesn't incur any additional error compared with exact IHT.
\end{remark}

\Cref{thm:iht} shows that by setting $k = \Theta(KL_{2k}^2/\alpha_{2k}^2), \eta = \alpha_{2k} / 2L_{2k}^2$, the parameter estimation can be improved geometrically when \algsubroute~ succeeds. We will show in \Cref{thm:batchsize} that with suffciently large batch size $m$, $\alpha_{2k}, L_{2k}$ are controlled and don't scale with $k, p$. When \algsubroute~fails, it can't make the $\Thetab$ estimation worse by too much. Given that success rate of \algsubroute~is controlled in \Cref{thm:top-recovery}, it naturally suggests that we can obtain the linear convergence in expectation. This leads to \Cref{thm:main2}. 

Define $\sigma_1^2 = \sigma_{GD}^2 + \sigma_{\Delta|GD}^2$, and $\sigma_2^2 = \sigma_{GD}^2 + \sigma_{Fail|GD}$. Let $\phi_t$ to be the success indicator of \algsubroute~ at time step $t$, and $\Phi_t = \left\{\phi_0, \phi_1, ..., \phi_t\right\}$. By \Cref{thm:top-recovery}, with $d = 48\log 2ck$, \algsubroute~ recovers top-$2k$ with probability at least $(1-{1/ c})$, we can easily show the convergence of \Cref{alg:qht} as

\begin{theorem}[Main result]\label{thm:main2}
    
	  Following the above notations, the expectation of the parameter recovery error of \Cref{alg:qht} is bounded by
	\begin{align*}
		&\EE_{B_t, \Phi_t}\left[\fnorm{\Thetab^{t}-\Thetab^\star}^2\right] \le \left(\kappa_1+{1\over c}\left(\kappa_2-\kappa_1\right)\right)^{t}\fnorm{\Thetab^{0}-\Thetab^\star}^2  \\
		& + \left[\left(\kappa_1+{1\over c}\left(\kappa_2-\kappa_1\right)\right)^{t}-1\right]\left({\sigma_1^2\over \kappa_1-1}\right) 
		+ {\kappa_2-1\over c-c\kappa_1 + \kappa_1 - \kappa_2}\left({\sigma_2^2\over\kappa_2-1}-{\sigma_1^2\over\kappa_1-1}\right).
	\end{align*}
	
\end{theorem}

This shows that \Cref{alg:qht} achieves linear convergence by setting $c \ge {(\kappa_2 - \kappa_1) / (1-\kappa_1)}$. With $c$ increasing, the error ball  converges to ${\sigma_1^2 / (1-\kappa_1)}$. The proof follows directly by taking expectation of the result we obtain in \Cref{thm:iht} with the recovery success probability established in \Cref{thm:top-recovery}.

\textbf{Computational analysis}
With the linear convergence, the computational complexity is dominated by the complexity per iteration. Before discussing the complexity, we first establish the dependency between $L_k, \alpha_k$ and $m$ in the special case of quadratic regression, where the link function is identity. Notice that similar results would hold for more general quadratic problems as well.
\begin{theorem}[Minimum batch size]\label{thm:batchsize}
	For feature vector $\xb \in \RR^p$, whose first $p-1$ coordinates are drawn i.i.d. from a bounded distribution, and the $p$-th coordinate is constant 1. W.l.o.g., we assume the first $p-1$ coordinates to be zero mean, variance 1 and bounded by $B$. With batch size $m \gtrsim \ kB\log p / {\epsilon^2}$
	we have $\alpha_{k} \ge 1-\epsilon$, $L_{k} \le 1+\epsilon$ with high probability.
\end{theorem}

Note that the sample complexity requirement matches the known information theoretic lower bound for recovering $k$-sparse $\Thetab$ up to a constant factor. The proof is similar to the analysis of restricted isometry property  in sparse recovery. Recall that by \Cref{thm:top-recovery}, we have the per-iteration complexity $\widetilde \Ocal(m(p+b))$. Combining the results of  \Cref{lemma:Frobenius}, \Cref{thm:main2,thm:batchsize}, we have the following corollary on the complexity:
\begin{corollary}[Achieving sub-quadratic space and time complexity]\label{coro:sub-quadratic}

 In the case of quadratic regression, by setting the parameters  as above, \algname~ recovers $\Thetab^\star$ in expectation up to a noise ball with linear convergence. The time and space complexity of \algname~ is $\widetilde \Ocal (k(k+p))$, which is sub-quadratic  when $k$ is $\Ocal(p^\gamma)$ for $\gamma < 1$.

\end{corollary}

Note that  the optimal time and space complexity is $\Omega(kp)$, since a minimum of $\Omega(k)$ samples are required for recovery, and $\Omega(p)$ for reading all entries. \Cref{coro:sub-quadratic} shows the time and space complexity of \algname~is $\widetilde \Ocal (k(k+p))$, which is nearly optimal.

\section{Synthetic Experiments} \label{sec:experiments}

To examine the sub-quadratic time and space complexity, we design three tasks to answer the following three questions: (i) Whether \Cref{alg:qht} maintains linear convergence   despite the hard thresholding not being accurate? (ii) What is the dependency between $b$ and $k$ to guarantee successful recovery? (iii) What is the dependency between $m$ and $p$ to guarantee successful recovery? 
Recall that the per-iteration complexity of \Cref{alg:qht} is $\widetilde O(m(p+b))$, where $b$ upper bounds the size of \algsubroute's output set, $p$ is the dimension of features and $m$ is batch size and $k$ is the sparsity of estimation.
It will be clear as we proceed how the three questions can support sub-quadratic complexity. 

\textbf{Experimental setting}  
We generate feature vectors $\xb_i$, whose coordinates follow i.i.d. uniform distribution on $[-1, 1]$. Constant $1$ is appended to each feature vector to model the linear terms and intercept. 
The true support is uniformly selected from all the interaction and linear terms, where the non-zero parameters are then generated uniformly on $[-20,-10]\cup [10, 20]$. 
Note that for the experiment concerning  minimum batch size $m$, we instead use Bernoulli distribution to generate both the features and the parameters, which reduces the variance for multiple random runs and makes our phase transition plot clearer. 
The output $y_i$s, are generated following $\xb_i^\top \Thetab^\star \xb_i$. 
On the algorithm side, by default, we set $p=200$, $d=3$, $K=20$, $k=3K$, $\eta = 0.2$. Support recovery results with different $b$-$K$ combinations are averaged over $3$ independent runs, results for $m$-$p$ combinations are averaged over $5$ independent runs. All experiments are terminated after 150 iterations.

\begin{figure}[ht]
	\begin{subfigure}{0.32\columnwidth}
		\centering
		\includegraphics[width=\linewidth]{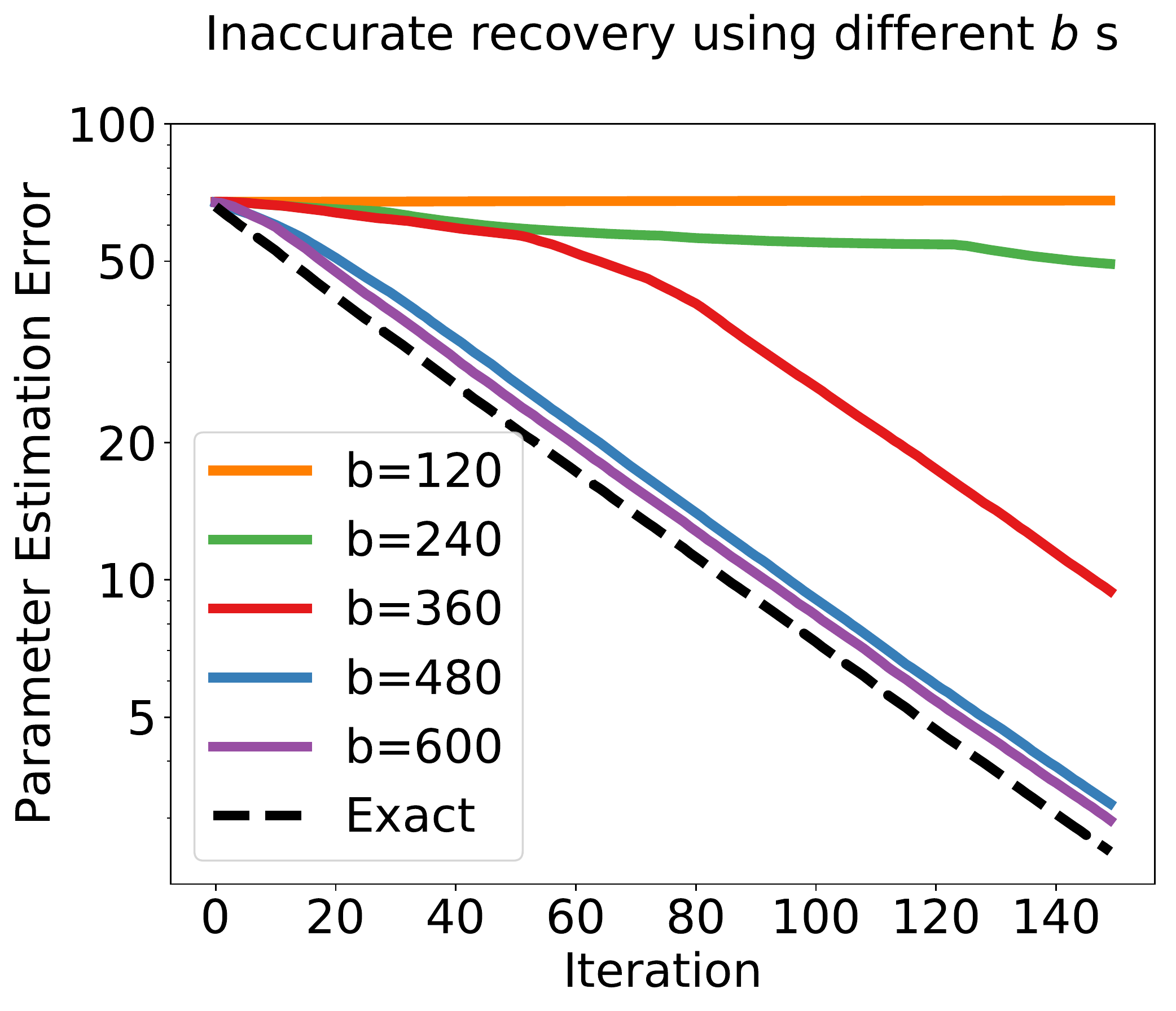}
		\caption*{(a) Inaccurate recovery using different \algsubroute's output set sizes $b$}
	\end{subfigure}
	\hfill 
	\begin{subfigure}{0.32\columnwidth}
		\centering
		\includegraphics[width=\linewidth]{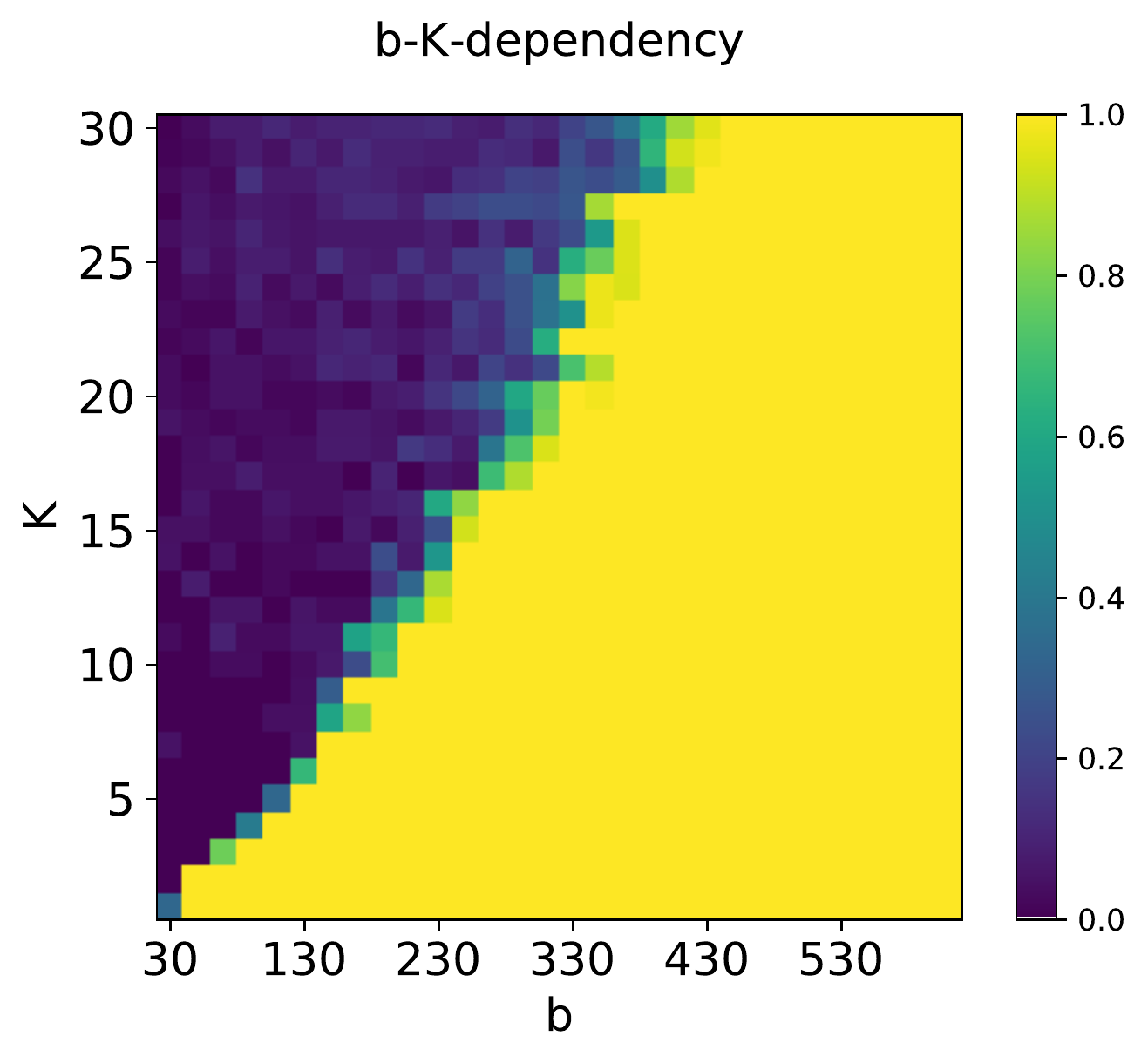}
		\caption*{(b) Support recovery results with different $b$ and $K$}
	\end{subfigure}
	\hfill 
	\begin{subfigure}{0.32\columnwidth}
		\centering
		\includegraphics[width=\linewidth]{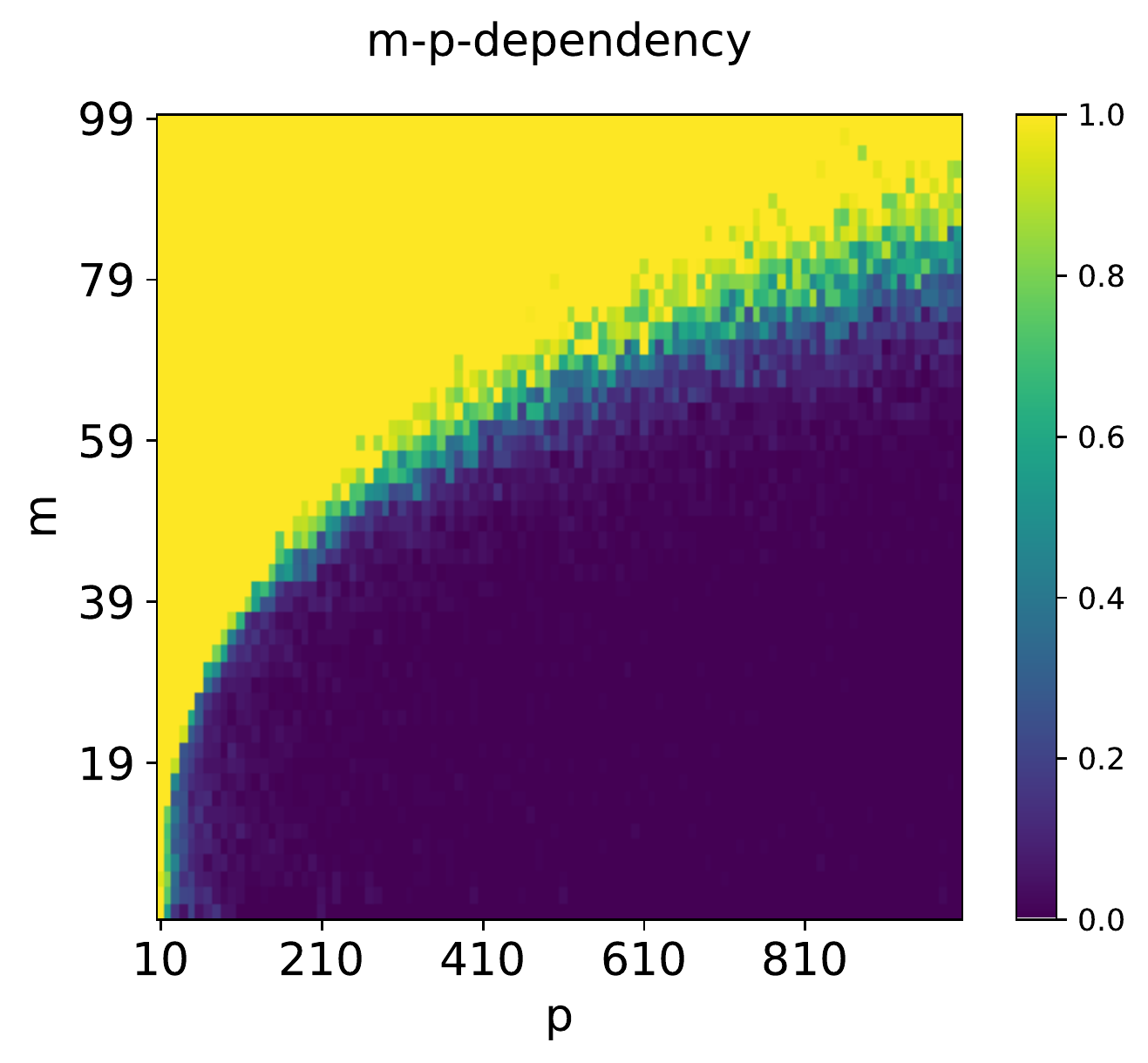}
		\caption*{(c) Support recovery results with different $m$ and $p$}
	\end{subfigure}

	\caption{Synthetic experiment results: note $b, m$ are the parameters we used for \algname and \algsubroute, where $b$ upper bounds the size of \algsubroute's output set and $m$ is the batch size used for \algname. Recall $p$ is the dimension of features and $K$ is the sparsity of $\Thetab^\star$. {\bf (a)} the convergence behavior with different choices of $b$. Linear convergence holds for small $b$, e.g., $360$, when the parameter space is around $20,000$. {\bf (b)} Support recovery results with different choices of $(b, K)$. We observe a linear dependence between $b$ and $K$.  {\bf (c)} Support recovery results with different choices of $(m, p)$. $m$ scales sub-linearly with $p$ to ensure a success recovery.}
	\label{fig:sim}
\end{figure}

\textbf{Inaccurate support recovery with different $b$'s} 
\Cref{fig:sim}-(a) demonstrates  different convergence results, measured by $\fnorm{\Thetab - \Thetab^\star}$ with multiple choices of $b$ for \algsubroute~ in \Cref{alg:qht}. 
The dashed curve is obtained by replacing  \algsubroute~  with exact top elements extraction (calculates the gradient exactly and picks the top elements).  
This is statistically optimal, but comes with quadratic complexity. 
By choosing a moderately large $b$, the inaccuracy induced by \algsubroute~ has negligible impact on the convergence. 
Therefore, \Cref{alg:qht} can maintain the linear convergence despite the support recovery in each iteration is inaccurate. This aligns with \Cref{thm:iht}. With linear convergence, the per iteration complexity will dominate the overall complexity.

\textbf{Dependency between $b$ and sparsity $k$} 
We proceed to see the proper choice of $b$ under different sparsity $k$ (we use $k=3K$). 
We vary the sparsity $K$ from $1$ to $30$, and apply \Cref{alg:qht} with $b$ ranges from $30$ to $600$.  As shown in \Cref{fig:sim}-(b), the minimum proper choice of $b$ scales no more than linearly with $k$. This agrees with our analysis in \Cref{thm:top-recovery}. The per-iteration complexity then collapse to $\widetilde\Ocal({m(p+k)})$.

\begin{wrapfigure}{r}{0.33\columnwidth}
	\centering
	\includegraphics[width=\linewidth]{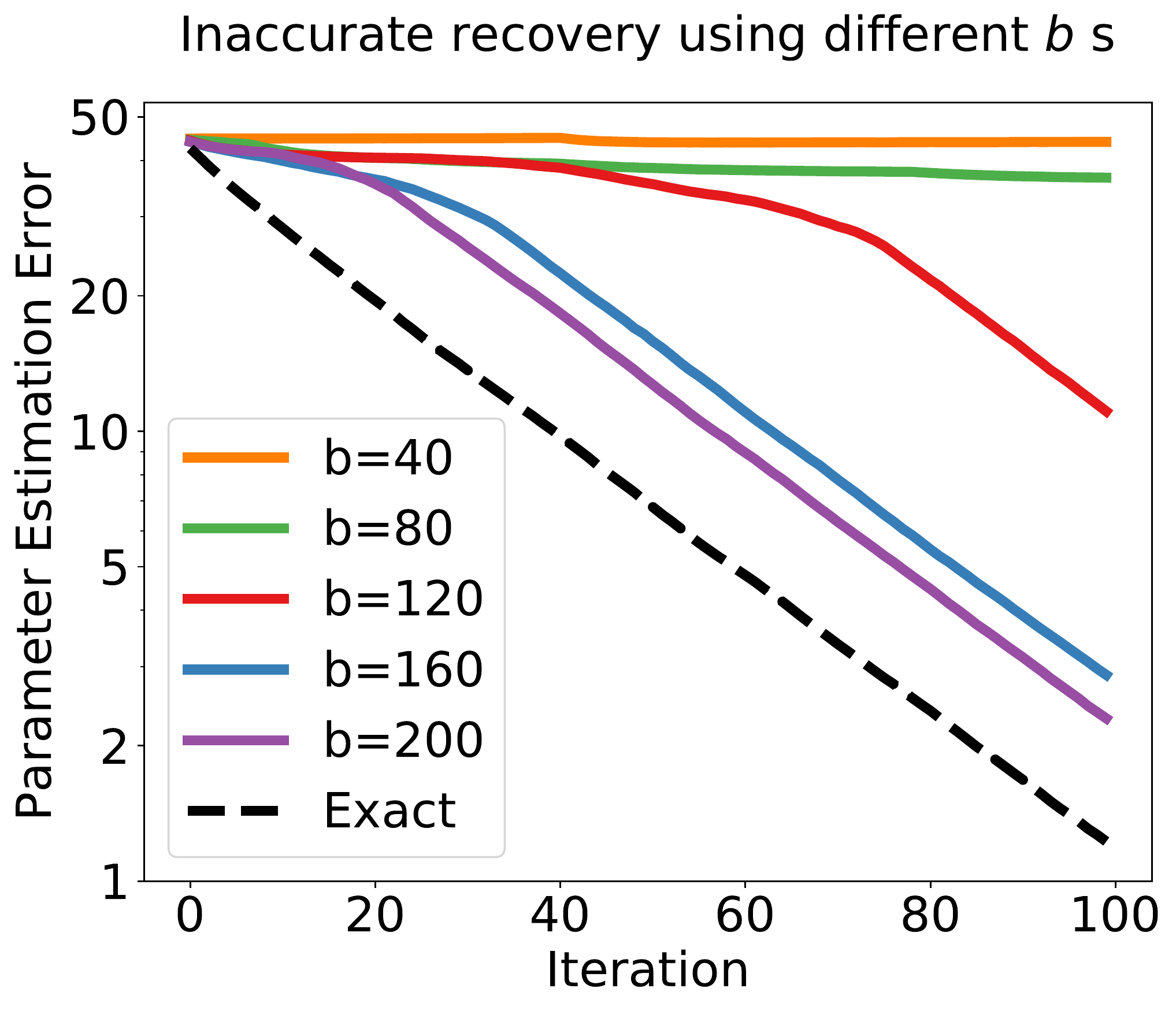}
	\caption{3-order regression support recovery using different \algsubroute's output set sizes $b$. The algorithm still remains linear convergence with proper setting of $b$.}
\label{fig:tensor}
\end{wrapfigure}

\textbf{Dependency between batch size $m$ and dimension $p$} 
Finally, we characterize the dependency between minimum batch size $m$ and the input dimension $p$. This will complete our discussion on the per-iteration complexity. 
The batch size varies from $1$ to $99$, and the  input dimension varies from $10$ to $1000$. 
In this experiment, we employ the \Cref{alg:qht} with \algsubroute~ replaced by exact top-$k$ elements extraction. 
\Cref{fig:sim}-(c) demonstrates the support recovery success rate of each $(k, p)$ combination. It shows  the minimum batch size scales in logarithm with dimension $p$, as we proved in \Cref{thm:batchsize}. Together with the previous experiment, it establishes the sub-quadratic complexity.

\textbf{Higher order interaction} 
\algname~ is also extensible to higher order interactions. Specifically, by exploiting similar gradient structure $\sum r_i\xb_i \otimes \xb_i \otimes \xb_i$, where $r_i$ denotes the residual for $(\Xb_i, y_i)$, $\otimes$ denotes the outer product of vector, we can again combine sketching with high-dimensional optimization to achieve nearly linear time and space (for constant sparsity).
\newline

For the experiment, we adopt the similar setting as for the \textbf{Inaccurate support recovery with different $b$'s} experiment. The main difference is that we change from $y_i = \xb_i^\top \Thetab^\star \xb_i$ to $y_i = \sum\Thetab_{i,j,k}\xb_i\xb_j\xb_k$, where $\Thetab$ is now a three dimension tensor. Further, we set the dimension of $\xb$ to $30$ and the sparsity $K = 20$. \Cref{fig:tensor} demonstrates the result of support recovering of 3-order interaction terms with different setting of $b$, where $b$ still bounds the size of \algsubroute's output set. We can see that \algname~still maintains the linear convergence in the higher order setting.

\section*{Acknowledgement}
We would like to acknowledge NSF grants 1302435 and
1564000 for supporting this research.

\clearpage 
\newpage 

\bibliography{ref}

\begin{thebibliography}{10}

\bibitem{abboud2017distributed}
Amir Abboud, Aviad Rubinstein, and Ryan Williams.
\newblock Distributed pcp theorems for hardness of approximation in p.
\newblock In {\em 2017 IEEE 58th Annual Symposium on Foundations of Computer
  Science (FOCS)}, pages 25--36. IEEE, 2017.

\bibitem{ballard2015diamond}
Grey Ballard, Tamara~G Kolda, Ali Pinar, and C~Seshadhri.
\newblock Diamond sampling for approximate maximum all-pairs dot-product (mad)
  search.
\newblock In {\em 2015 IEEE International Conference on Data Mining}, pages
  11--20. IEEE, 2015.

\bibitem{bien2013lasso}
Jacob Bien, Jonathan Taylor, and Robert Tibshirani.
\newblock A lasso for hierarchical interactions.
\newblock {\em Annals of statistics}, 41(3):1111, 2013.

\bibitem{blumensath2009iterative}
Thomas Blumensath and Mike~E Davies.
\newblock Iterative hard thresholding for compressed sensing.
\newblock {\em Applied and computational harmonic analysis}, 27(3):265--274,
  2009.

\bibitem{chen2018hardness}
Lijie Chen.
\newblock On the hardness of approximate and exact (bichromatic) maximum inner
  product.
\newblock In {\em 33rd Computational Complexity Conference (CCC 2018)}. Schloss
  Dagstuhl-Leibniz-Zentrum fuer Informatik, 2018.

\bibitem{chen2001atomic}
Scott~Shaobing Chen, David~L Donoho, and Michael~A Saunders.
\newblock Atomic decomposition by basis pursuit.
\newblock {\em SIAM review}, 43(1):129--159, 2001.

\bibitem{choi2010variable}
Nam~Hee Choi, William Li, and Ji~Zhu.
\newblock Variable selection with the strong heredity constraint and its oracle
  property.
\newblock {\em Journal of the American Statistical Association},
  105(489):354--364, 2010.

\bibitem{efron2004least}
Bradley Efron, Trevor Hastie, Iain Johnstone, Robert Tibshirani, et~al.
\newblock Least angle regression.
\newblock {\em The Annals of statistics}, 32(2):407--499, 2004.

\bibitem{hao2018sparse}
Botao Hao, Anru Zhang, and Guang Cheng.
\newblock Sparse and low-rank tensor estimation via cubic sketchings.
\newblock {\em arXiv preprint arXiv:1801.09326}, 2018.

\bibitem{hao2018model}
Ning Hao, Yang Feng, and Hao~Helen Zhang.
\newblock Model selection for high-dimensional quadratic regression via
  regularization.
\newblock {\em Journal of the American Statistical Association},
  113(522):615--625, 2018.

\bibitem{hao2014interaction}
Ning Hao and Hao~Helen Zhang.
\newblock Interaction screening for ultrahigh-dimensional data.
\newblock {\em Journal of the American Statistical Association},
  109(507):1285--1301, 2014.

\bibitem{jain2014iterative}
Prateek Jain, Ambuj Tewari, and Purushottam Kar.
\newblock On iterative hard thresholding methods for high-dimensional
  m-estimation.
\newblock In {\em Advances in Neural Information Processing Systems}, pages
  685--693, 2014.

\bibitem{kocaoglu2014sparse}
Murat Kocaoglu, Karthikeyan Shanmugam, Alexandros~G Dimakis, and Adam Klivans.
\newblock Sparse polynomial learning and graph sketching.
\newblock In {\em Advances in Neural Information Processing Systems}, pages
  3122--3130, 2014.

\bibitem{li2016nonconvex}
Xingguo Li, Raman Arora, Han Liu, Jarvis Haupt, and Tuo Zhao.
\newblock Nonconvex sparse learning via stochastic optimization with
  progressive variance reduction.
\newblock {\em arXiv preprint arXiv:1605.02711}, 2016.

\bibitem{Li2015ModelingGI}
Yun Li, George~T. O'Connor, Jos{\'e}e Dupuis, and Eric~D. Kolaczyk.
\newblock Modeling gene-covariate interactions in sparse regression with group
  structure for genome-wide association studies.
\newblock {\em Statistical applications in genetics and molecular biology}, 14
  3:265--77, 2015.

\bibitem{lim2015learning}
Michael Lim and Trevor Hastie.
\newblock Learning interactions via hierarchical group-lasso regularization.
\newblock {\em Journal of Computational and Graphical Statistics},
  24(3):627--654, 2015.

\bibitem{liu2018high}
Liu Liu, Yanyao Shen, Tianyang Li, and Constantine Caramanis.
\newblock High dimensional robust sparse regression.
\newblock {\em arXiv preprint arXiv:1805.11643}, 2018.

\bibitem{mansour1995randomized}
Yishay Mansour.
\newblock Randomized interpolation and approximation of sparse polynomials.
\newblock {\em SIAM Journal on Computing}, 24(2):357--368, 1995.

\bibitem{murata2018sample}
Tomoya Murata and Taiji Suzuki.
\newblock Sample efficient stochastic gradient iterative hard thresholding
  method for stochastic sparse linear regression with limited attribute
  observation.
\newblock In {\em Advances in Neural Information Processing Systems}, pages
  5317--5326, 2018.

\bibitem{nguyen2017linear}
Nam Nguyen, Deanna Needell, and Tina Woolf.
\newblock Linear convergence of stochastic iterative greedy algorithms with
  sparse constraints.
\newblock {\em IEEE Transactions on Information Theory}, 63(11):6869--6895,
  2017.

\bibitem{pagh2013compressed}
Rasmus Pagh.
\newblock Compressed matrix multiplication.
\newblock {\em ACM Transactions on Computation Theory (TOCT)}, 5(3):9, 2013.

\bibitem{rahimi2008random}
Ali Rahimi and Benjamin Recht.
\newblock Random features for large-scale kernel machines.
\newblock In {\em Advances in neural information processing systems}, pages
  1177--1184, 2008.

\bibitem{rendle2010factorization}
Steffen Rendle.
\newblock Factorization machines.
\newblock In {\em 2010 IEEE International Conference on Data Mining}, pages
  995--1000. IEEE, 2010.

\bibitem{shawe2004kernel}
John Shawe-Taylor, Nello Cristianini, et~al.
\newblock {\em Kernel methods for pattern analysis}.
\newblock Cambridge university press, 2004.

\bibitem{she2018group}
Yiyuan She, Zhifeng Wang, and He~Jiang.
\newblock Group regularized estimation under structural hierarchy.
\newblock {\em Journal of the American Statistical Association},
  113(521):445--454, 2018.

\bibitem{shen2017tight}
Jie Shen and Ping Li.
\newblock A tight bound of hard thresholding.
\newblock {\em The Journal of Machine Learning Research}, 18(1):7650--7691,
  2017.

\bibitem{shrivastava2014asymmetric}
Anshumali Shrivastava and Ping Li.
\newblock Asymmetric lsh (alsh) for sublinear time maximum inner product search
  (mips).
\newblock In {\em Advances in Neural Information Processing Systems}, pages
  2321--2329, 2014.

\bibitem{sipser1994expander}
Michael Sipser and Daniel~A Spielman.
\newblock Expander codes.
\newblock In {\em Proceedings 35th Annual Symposium on Foundations of Computer
  Science}, pages 566--576. IEEE, 1994.

\bibitem{tibshirani1996regression}
Robert Tibshirani.
\newblock Regression shrinkage and selection via the lasso.
\newblock {\em Journal of the Royal Statistical Society: Series B
  (Methodological)}, 58(1):267--288, 1996.

\bibitem{vershynin2010introduction}
Roman Vershynin.
\newblock Introduction to the non-asymptotic analysis of random matrices.
\newblock {\em arXiv preprint arXiv:1011.3027}, 2010.

\bibitem{williams2018difference}
Ryan Williams.
\newblock On the difference between closest, furthest, and orthogonal pairs:
  Nearly-linear vs barely-subquadratic complexity.
\newblock In {\em Proceedings of the Twenty-Ninth Annual ACM-SIAM Symposium on
  Discrete Algorithms}, pages 1207--1215. Society for Industrial and Applied
  Mathematics, 2018.

\bibitem{wu2010screen}
Jing Wu, Bernie Devlin, Steven Ringquist, Massimo Trucco, and Kathryn Roeder.
\newblock Screen and clean: a tool for identifying interactions in genome-wide
  association studies.
\newblock {\em Genetic Epidemiology: The Official Publication of the
  International Genetic Epidemiology Society}, 34(3):275--285, 2010.

\bibitem{yu2017greedy}
Hsiang-Fu Yu, Cho-Jui Hsieh, Qi~Lei, and Inderjit~S Dhillon.
\newblock A greedy approach for budgeted maximum inner product search.
\newblock In {\em Advances in Neural Information Processing Systems}, pages
  5453--5462, 2017.

\bibitem{yu2016learning}
Rose Yu and Yan Liu.
\newblock Learning from multiway data: Simple and efficient tensor regression.
\newblock In {\em International Conference on Machine Learning}, pages
  373--381, 2016.

\end{thebibliography}
\bibliographystyle{plain}

\clearpage
\newpage

\appendix

\section{Details of \algsubroute}
\label{sec:atee-formal}

In this section, we provide the formal algorithm for \algsubroute, stated in \Cref{alg:sketch-formal}. \algsubroute~ consists of two sub-routines: an efficient sketching operation (\textcolor{\bcolor}{line 5-11}), and  an efficient extraction operation (\textcolor{\bcolor}{line 12-15}). 

\begin{algorithm}[h]
\centering 
\begin{algorithmic}[1]
\STATE \textbf{Input:} Matrix $\Ab$, matrix $\Bb$, top selection size $k$
\STATE \textbf{Parameters:}  Output set size limit $b$, repetition number $d$, significant level $\Delta$
\STATE \textbf{Expected Output: } A set $\Lambda$, which is the top-$k$ elements in $\Ab\Bb^\top$ whose absolute value is also greater than $\Delta$ 
\STATE \textbf{Output:} Set $\widetilde\Lambda$ of indices, with size at most $b$ and approximately contains $\Lambda$
 \FOR{$t=0$ \textbf{to} $d-1$}  
 \STATE Construct expander code table $\Eb \in \{0, 1\}^{p\times l}$. Let $\eb_r$ be the $r$-th column of $\Eb$
\STATE $\Ib_{\eb_r} = \diag(\eb_r) $, $\Cb_r =\Ib_{\eb_r} \Ab, \Cb_{r+l} = \Ib_{\eb_r} \Bb$, $\forall r \in [l]$. Init $\Sbb$ as a $2l\times b$ matrix
\STATE Generate pairwise independent hash functions $h_1, h_2: [p] \rightarrow [b]$
\FOR{$r=0$ \textbf{to} $l-1$}
\STATE $\Sbb[r,:]=\textsc{Compressed-Product}(\Cb_r, \Bb, b, h_1, h_2)$
\STATE $\Sbb[r+l, :]=\textsc{Compressed-Product}(\Ab, \Cb_{r+l}, b, h_1, h_2)$  
\ENDFOR
\FOR{$q=0$ \textbf{to} $b-1$}
\STATE $\ob^{q} = \one(\Sbb[:, q] > \frac{\Delta}{2})$
\STATE $(i,j) = \textsc{Decode}(\ob^{q};~\Eb)$
\STATE $S = S\cup \{(i,j)\}$
\ENDFOR 
\ENDFOR 
\STATE \textbf{Return:} $\{ (i,j) | \# (i,j)\in S \ge \frac{d}{2} \}$ 
\\\hrulefill 
\STATE \textbf{function}~ \textsc{Compressed-Product}$(\Ab, \Bb, b, h_1, h_2)$:

\STATE ~~~~ Generate random sign functions $s_1, s_2: [p]\rightarrow \{-1, +1\}$ , column length of $\Ab, \Bb$ are $m$
\STATE	 ~~~~ \textbf{for} $i=0$ \textbf{to} $m-1$ \textbf{do}
\STATE ~~~~~~~~ $\pb_{\ab_i} \leftarrow $\textsc{Count-Sketch}$(\ab_i, h_1, s_1, b)$ , $\pb_{\bb_i} \leftarrow $\textsc{Count-Sketch}$(\bb_i, h_2, s_2, b)$
\STATE ~~~~~~~~  $\sbb_i = \mathtt{IFFT}\left( \mathtt{FFT}\left( \pb_{\ab_i} \right) \circ \mathtt{FFT}\left(  \pb_{\bb_i} \right) \right)$
\STATE ~~~~ \textbf{Return:} $\sbb = \sum_{i=0}^{m-1} \sbb_i$ 
\\\hrulefill
\STATE \textbf{function}~ \textsc{Count-Sketch}$(\xb, h, s, b)$
\STATE ~~~~ Init $\pb$ as a length $b$ vector
\STATE ~~~~ \textbf{{for} $i=0$ \textbf{to} $p-1$} \textbf{do}
\STATE ~~~~~~~~ $\pb[h(i)]=\pb[h(i)] + s(i)\xb[i]$
\STATE ~~~~ \textbf{Return:} $\pb$
\end{algorithmic}
\caption{\textsc{Approximate Top Elements Extraction (\algsubroute-Formal)} }
\label{alg:sketch-formal}
\end{algorithm}

For the sketching part, the algorithm first generate expander code which maps $[p]$ to $\{0,1\}^l$, where $l$ is the length of the codeword. Based on this encoding, we construct a table $\Eb \in \{0, 1\}^{p\times l}$ where the $i$-th row is the codeword which encoded from $i$. Denote $\eb_r$ as the $r$-th column of $\Eb$, we construct diagonal matrix $\Ib_{\eb_r}=\text{diag}(\eb_r)$. Then $2l$ different sub-matrices of $\Ab$ and $\Bb$ is constructed by $\Cb_r = \Ib_{\eb_r}\Ab, \Cb_{r+l}=\Ib_{\eb_r}\Bb, \forall r \in \left[l\right]$. It then sketches $\Ab \Cb_{r+l}^\top$ and $\Cb_r\Bb$, each of the matrices into a $b$ length vector, where $b \ll p^2$. The result is stored in $\Sbb \in \RR^{2l\times b}$. By exploiting the factorization of this matrix, the matrix outer product can be sketched in $\Ocal({b\log(b)})$ using fast Fourier transform (FFT) as used in  \cite{pagh2013compressed}(\textcolor{\bcolor}{line 18 - 21}).

For the sub-linear extraction, we first binarify the matrix $\Sbb$ with threshold $\Delta / 2$. Then each column of $\Sbb$ becomes a codeword, where the first $l$ bits encodes the row index of the elements whose absolute value is greater than $\Delta / 2$, the last $l$ bits encode its column index. By using expander code, it takes linear time $\Ocal(l)$ to finish decoding. The whole process will be repeated for $d$ times and only the elements that are recovered for more than $d/2$ times will be recorded for output. This can boost the success probability of top-$k$ support recovery.

\section{\algname-VR Algorithm} \label{appendix:SVRG}

The application of \algname~ to SVRG follows a very similar path as we apply it to SGD. The only trick is to ultilize the linearity of sketching. For SVRG, we will generate the hash function $h_1, h_2, s_1, s_2$ as described in SGD case at the begining of each outer iteration. They will be kept same through all the inner iterations. The skeching result $\sbb^i$ of the full gradient at the beginning of $i$-th outer loop will add up with the sketching result $\widetilde\sbb^i_j$ of the corresponding $j$-th inner loop. The summation $\sbb^i_j$ then goes through the decoding process, which is the same as SGD. For psedocode, see \Cref{alg:qht-svrg}.

\newcommand{\outerindex}{i}
\newcommand{\innerindex}{j}
\newcommand{\outerround}{T}
\newcommand{\innerround}{t}

\begin{algorithm}[htbp]
\centering 
\begin{algorithmic}[1]
\STATE \textbf{Input:} Dataset $\{\xb_{i'}, y_{i'}\}_{i'=0}^{n-1}$,  threshold $\Delta$, dimension $p$, outer / inner round number $\outerround / \innerround$
\STATE \textbf{Output:}  $\widehat{\Thetab}$
\STATE \textbf{Parameters:} Codeword length $l$, sketch size $b$, repetition number $d$
\STATE Initialize $\Thetab^0$ as a $p\times p$ zero matrix. 
\FOR{$\outerindex=0$ \textbf{to} $\outerround-1$}
\STATE $\Thetab^\outerindex_0 = \Thetab^\outerindex$, construct expander code table $\Eb^i \in \{0, 1\}^{p\times l}$. 
\STATE Generate hash functions $h_1, h_2: \left[p\right] \rightarrow \left[b\right]$, and $s_1, s_2: \left[p\right]\rightarrow \{-1,1\}$
\STATE  $\sbb^\outerindex = \textsc{Interaction-Sketch}(\{\xb_{i'}, u_{i'}\}_{i'=0}^{n-1}, \Eb^i, h_1, h_2, s_1, s_2)$
\STATE $\Gb^i := \frac{1}{n}\sum_{i'=0}^{n-1} u_{i'}\xb_{i'}\xb_{i'}^\top$ 
\FOR{$\innerindex=0$ \textbf{to} $\innerround-1$}
\STATE Randomly pick a sample (min-batch) $\Bcal_\innerindex^\outerindex$
\STATE $\widetilde{u}_{i'} = u_{i'}(\Thetab_\innerindex^{\outerindex}, \xb_{i'}, y_{i'}) - u_{i'}(\Thetab_0^\outerindex, \xb_{i'}, y_{i'})$, $\forall i'\in \Bcal_\innerindex^\outerindex$, $\widetilde{\Gb}_\innerindex^\outerindex := \frac{1}{|\Bcal_{\innerindex}^{\outerindex}|}\sum_{i\in \Bcal_\innerindex^\outerindex} \widetilde{u}_{i'} \xb_{i'}\xb_{i'}^\top$
\STATE  $\widetilde{\sbb}^\outerindex_\innerindex= \textsc{Interaction-Sketch}(\{\xb_{i'}, \widetilde{u}_{i'}\}_{i'\in \Bcal_\innerindex^\outerindex}, \Eb^i, h_1, h_2, s_1, s_2)$
\STATE $\sbb^\outerindex_\innerindex = \sbb^\outerindex + \widetilde{\sbb}^\outerindex_\innerindex$
\STATE $\widetilde{S}^\outerindex_\innerindex = \textsc{Interaction-Decode}(\sbb^\outerindex_\innerindex, \Delta, \Eb^i)$ 
\STATE $S^\outerindex_\innerindex = \widetilde{S}^\outerindex_\innerindex\cup \supp(\Thetab^\outerindex_\innerindex)$
\STATE $\Pcal_{S^\outerindex_\innerindex}(\Gb^\outerindex_\innerindex)\leftarrow$ the gradient value $\Gb_j^i := \Gb^\outerindex + \widetilde{\Gb}_{\innerindex}^{\outerindex}$ calculated only on index set $S^\outerindex_\innerindex$, 
\STATE $\Thetab^{\outerindex}_{\innerindex+1} = \Hcal_{k}\left( \Thetab^{\outerindex}_{\innerindex} - \eta \Pcal_{S^\outerindex_\innerindex}(\Gb^\outerindex_\innerindex)  \right)$
\ENDFOR 
\STATE $\Thetab^{\outerindex+1} = \Thetab^{\outerindex}_{\innerindex'}$, for $\innerindex'\sim \mathtt{Unif}(\{0,\cdots, \innerround -1\})$
\ENDFOR 
\STATE \textbf{Return:} $\widehat{\Thetab}= \Thetab^{\outerround}$ 
\\\hrulefill 
\STATE \textbf{function} \textsc{Interaction-Sketch}$(\{\xb_i, u_i\}_{i\in \Bcal}, \Eb, h_1, h_2, s_1, s_2):$
\STATE ~~~~ Set $\Ab\in\RR^{p\times |\Bcal|}$, where each column of $\Ab$($\Bb$) is $u_i\xb_i$, $i\in \Bcal$. 
\STATE ~~~~ Set $\Bb\in\RR^{p\times |\Bcal|}$, where each column of $\Ab$($\Bb$) is $\xb_i$, $i\in \Bcal$. 
\STATE ~~~~ Set $\sbb$ as a $d\times 2l\times b$ tensor.
 \STATE ~~~~ \textbf{{for} $t=0$ \textbf{to} $d-1$}  \textbf{do}
 \STATE ~~~~~~~~ Let $\eb_r$ be the $r$-th column of $\Eb$, $\Ib_{\eb_r} = \diag(\eb_r) $, $\Cb_r =\Ib_{\eb_r} \Ab, \Cb_{r+l} = \Ib_{\eb_r} \Bb$, $\forall r \in [l]$.
\STATE ~~~~~~~~ \textbf{{for} $r=0$ \textbf{to} $l-1$}\textbf{do}
\STATE ~~~~~~~~~~~~ $\sbb[t, r,:]=\textsc{Compressed-Product}(\Cb_r, \Bb, b, h_1, h_2)$ 
\STATE ~~~~~~~~~~~~ $\sbb[t, r+l, :]=\textsc{Compressed-Product}(\Ab, \Cb_{r+l}, b, h_1, h_2)$  
\STATE ~~~~ \textbf{Return:} $\sbb$
\\\hrulefill 
\STATE \textbf{function} \textsc{Interaction-Decode}$(\sbb, \Delta, \Eb):$
\STATE ~~~~ Set $S = \emptyset$
\STATE ~~~~ \textbf{{for} $t=0$ \textbf{to} $d-1$ } \textbf{do}
\STATE ~~~~~~~~ \textbf{{for} $q=0$ \textbf{to} $b-1$} \textbf{do}
\STATE ~~~~~~~~~~~~ $\ob^{t,q} = \one(\sbb^{t, :, q} > \frac{\Delta}{2})$ 
\STATE ~~~~~~~~~~~~ $(i,j) = \textsc{Decode}(\ob^{t,q}, \Eb)$
\STATE ~~~~~~~~~~~~ $S = S\cup \{(i,j)\}$
\STATE ~~~~ \textbf{Return:} $\{ (i,j) | \# (i,j)\in S \ge \frac{d}{2} \}$
\end{algorithmic}
\caption{\textsc{\algname~ with variance reduction (\algname-VR)}}
\label{alg:qht-svrg}
\end{algorithm}

\section{\algname-VR Analysis}\label{sec:extension}

Here we proceed to provide theoretical guarantee for \Cref{appendix:SVRG}.

Similar to the definitions for \Cref{thm:iht} and \Cref{thm:main2}, we define $\Lambda_{2k}$ to be the set of top $2k$ elements in $\Gb^i_{j+1}$, $\Lambda_\Delta$ to be the set of elements in $\Gb^i_{j+1}$ whose magnitude is greater than $\Delta$ and the output set of ATEE to be $\widetilde\Lambda$. We have the support of interest $\Lambda = \Lambda_{2k}\cap\Lambda_{\Delta}$ and the number of top-$2k$ elements whose magnitude is below $\Delta$ $k_\Delta =|\Lambda_{2k}\backslash\Lambda_{\Delta}|$. Define $\Bcal_i$ to be the set of samples used during $i$-th outer loop. Recall that $\eta$ is the step size and $m$ is the batch size. We then have the following result:
\begin{theorem}[Per-round Convergence of \Cref{alg:qht-svrg}]\label{thm:SVRGPerRound} If $\Lambda\subseteq\widetilde\Lambda$, the per-round convergence of \Cref{alg:qht-svrg} is as follows:
	\begin{align*}
		\EE_{\Bcal_i}\left[F(\Thetab^{i+1}_0)-F(\Thetab^\star))\right] \le \kappa_{SVRG}\left[F(\Thetab^{i}_0)-F(\Thetab^\star)\right]+\sigma_{SVRG}
	\end{align*}
	where $\nu = 1+\frac{\rho + \sqrt{(4+\rho)\rho}}{2}, \rho={K}/{k}$, $\innerround$ is the inner round number, and 
	\begin{align*}
		\kappa_{SVRG} &= {1\over\alpha_{2k}\nu\eta(1-2\eta L_{2k})\innerround}+{2\eta L_{2k}\over1-2\eta L_{2k}}, \\
		\sigma_{SVRG} &= {4\nu\eta \sigma'(4L_{2k}\sqrt{k}\omega+\sigma')\innerround + {2\over \alpha_{2k}}\sqrt{k}\omega\sigma'+ \innerround\sigma^2_{\Delta|SVRG}\over  2 \nu\eta(1-2\eta L_{2k})\innerround}, \\
		\sigma' = \max_{|\Omega|=3k+K}&\fnorm{\Pcal_{\Omega}F(\Thetab^*)},~~~\sigma^2_{\Delta|SVRG}=4\sqrt{k_\Delta}\eta \sqrt{k}\omega\Delta + 2k_\Delta\eta^2\Delta^2. 
	\end{align*}
	To ensure the convergence, it requires that
	\begin{align*}
		\eta < {1\over 4L_{2k}},~~~\nu < {1\over 1-\eta\alpha_{2k}}.
	\end{align*}
\end{theorem}
The proof can be found in \Cref{prf:svrgPerRound}.
\begin{remark}
	Similar to the \Cref{rm:SGDsigma} case, $\sigma'$ is statistical error, which in noiseless case are 0. In the case that the magnitude of top-$2k$ elements in the gradient are all greater than $\Delta$, we have $\Lambda_{2k}\subseteq\Lambda_{\Delta}$ and $k_\Delta = 0$, which implies $\sigma_{\Delta|SVRG}=0$.
\end{remark}
To obtain the convergence result over all iterations, we adopt the same definition and assumption as in \Cref{thm:SVRGPerRound}.  By setting $c=\Theta({p})$, $d = 48\log(ck)$, we have that the inner loop of \Cref{alg:qht-svrg} succeeds with high probability (recall that $c$ was used to control the failure probability in \Cref{thm:top-recovery}, and it is not hard to see that the property in  \Cref{thm:top-recovery} still holds for \algname-VR). Then we have the following result: 
\begin{theorem}[Convergence of \algname-VR]\label{thm:mainSVRG} Under the same parameter setting as in \Cref{thm:top-recovery}, with $d$ specifically defined as above, 
	the convergence of \Cref{alg:qht-svrg} is given by
	\begin{align*}
		\EE_{B_t}\left[F(\Thetab^{t}_0)-F(\Thetab^\star))\right] \le \kappa_{SVRG}^s\left[F(\Thetab^0_0)-F(\Thetab^\star)\right]+{\sigma_{SVRG}\over 1-\kappa_{SVRG}},
	\end{align*}
    where the definitions of $\kappa_{SVRG}$ and $\sigma_{SVRG}$ follows from \Cref{thm:SVRGPerRound}. 
\end{theorem}
\begin{proof}
	Given that \algsubroute~ succeeds with high probability, the contraction of each iteration is characterized by \Cref{thm:SVRGPerRound}. By solving the recursion, we have the desired convergence.
\end{proof}

\begin{remark}
	Here we set \algsubroute~ to succeed with high probability, where in \Cref{thm:main2} it only requires \algsubroute~ to succeed with constant probability. This is because in each inner loop, the iterations share the same hash function $s, h$ as specified in \algsubroute, which removes the independence of \algsubroute~ for each iteration. Intuitively,  once \algsubroute~ fails, it could fail on the entire inner loop and ruined the estimation for $\Thetab$. By setting $c=\Theta({p})$, the high probability statement can be obtained without incurring more than $\Ocal(\log p)$ factor  higher complexity.
\end{remark}

\section{Technical Lemmas and Corollaries}\label{sec:analysis}

\begin{lemma}[Tight Bound for Hard Thresholding \cite{shen2017tight}] \label{lemma:tightbound}
Let $\Bb\in \RR^{p\times p}$ be an arbitrary matrix and $\Thetab \in\RR^{p\times p}$ be any $K$-sparse signal. For any $k\ge K$, we have the following bound:
\begin{align*}
    \fnorm{\Hcal_k(\Bb) - \Thetab} \le \sqrt{\nu} \fnorm{\Bb - \Thetab}, ~~~~ \nu = 1+\frac{\rho + \sqrt{(4+\rho)\rho}}{2}, ~~~~ \rho=\frac{\min\{K, p^2-k\} }{k-K+\min\{K, p^2-k\} }
\end{align*}
\end{lemma}
The provide a short proof in \Cref{prf:tightbound}.

\begin{corollary}[similar to co-coercivity]\label{coro:co-cercivity}
For a given support set $\Omega$, assume that the continuous function $f(\cdot)$ is $L_{|\Omega|}$-RSS and $K$-RC. Then, for all matrices $\Thetab, \Thetab'$ with $|\supp(\Thetab-\Thetab') \cup \Omega| \le K$, 
\begin{align*}
    \fnorm{ \Pcal_{\Omega}\left( \nabla f(\Thetab') - \nabla f(\Thetab) \right)}^2 \le 2L_{|\Omega|} \left( f(\Thetab') - f(\Thetab) - \minnerprod{ \nabla f(\Thetab), \Thetab'-\Thetab } \right).
\end{align*}
\end{corollary}
The proof can be found in \Cref{prf:co-cercivity}.

\begin{corollary}[bounding $\|\Pcal_{\Omega}(\Gb^t) \|_2^2$] \label{coro:proj-grad}Let $\Omega = \supp(\Thetab^{t-1}) \cup \supp(\Thetab^{t}) \cup \supp(\Thetab^\star)$. 
For SGD and SVRG, we have the following bound:
\begin{enumerate}
	\item SGD: $\Gb^t=\nabla f_{\iota_t} \left(\Thetab^{t-1}\right)$
	\begin{align*}
    	\EE_{\iota_t}\fnorm{\Pcal_{\Omega}(\Gb^t) }^2 \le 2L_{2k}^2\fnorm{\Thetab^{t-1}-\Thetab^\star}^2+2\fnorm{\Pcal_\Omega\left(\nabla f_{\iota_t}(\Thetab^\star)\right)}^2
    \end{align*}
	\item SVRG: $\Gb^i_j=\nabla f_{b_j} \left(\Thetab^{i}_j\right)-\nabla f_{b_j} \left(\Thetab^{i}_0\right)+\nabla F \left(\Thetab^{i}_0\right)$
	\begin{align*}
		\EE_{b_j}\fnorm{\Pcal_\Omega(\Gb^i_j)}^2 \le & 4L_{2k}\left[F(\Thetab^i_j)-F(\Thetab^\star)\right] + 4L_{2k}\left[F(\Thetab^i_0)-F(\Thetab^\star)\right] \\
		& -4L_{2k}\minnerprod{\nabla F(\Thetab^\star), \Thetab^i_j+\Thetab^i_0-2\Thetab^\star} + 4\fnorm{\Pcal_\Omega(\nabla F(\Thetab^\star))}
	\end{align*}
\end{enumerate}   
\end{corollary}
The proof can be found in \Cref{prf:proj-grad}.

\begin{corollary}[HT property]\label{coro:HTproperty}
Let $\Lambda_{2k}$ be the support of the top-$2k$ entries in $\Gb$ with largest absolute value, for a $k$-sparse matrix $\Thetab$,
\begin{align*}
	\Hcal_k \left( \Thetab - \eta \Gb \right) = \Hcal_k\left(\Thetab - \eta \Pcal_{\supp(\Thetab) \cup \Lambda_{2k}} (\Gb) \right)
\end{align*}
\end{corollary}
The proof can be found in \Cref{prf:HTproperty}.

\begin{corollary}[$\Delta$-Inexact Hard Thresholding]\label{coro:deltaError}
	Define $\Lambda_\Delta$ to be the set of elements in $\Gb_t$ whose magnitude is greater than $\Delta$. Further define $\Lambda = \Lambda_{2k}\cap\Lambda_{\Delta}, k_\Delta =|\Lambda_{2k}\backslash\Lambda_{\Delta}|$. Define,
	\begin{align*}
		\widetilde\Thetab^+ &= \Hcal_{k}\left(\Thetab-\eta\Gb_t\right)\\
		\Thetab^+ & = \Hcal_{k}\left(\Thetab-\eta\Pcal_{\widetilde\Lambda\cup\supp({\Thetab})}(\Gb_t)\right)
	\end{align*}
	In the case $\Lambda_\Delta\subseteq\widetilde\Lambda$ and $\Lambda_\Delta\subseteq\widetilde\Lambda$, we have the bound,
	\begin{align*}
		\norm{\Thetab^+ - \widetilde\Thetab^+}_F & \le \eta\Delta\sqrt{2k_\Delta}
	\end{align*}
	
\end{corollary}
The proof can be found in \Cref{prf:deltaError}.

\section{Proofs for \Cref{sec:guarantees}} \label{prf:guarantees}

\subsection{Proof of \Cref{thm:top-recovery}}\label{prf:sketch}
\begin{proof}
The proof of \Cref{thm:top-recovery} heavily relies on the analysis in \cite{pagh2013compressed}. 
Given that $\nabla f_{\Bcal_i}(\Thetab)$ can be expressed as multiplication of two matrices, we slightly abuse the notation $\Ab, \Bb$ to denote the pair of matrices that $\Ab\Bb^\top = \nabla f_{\Bcal_i}(\Thetab)$.

Denote the output of \textsc{Compressed-Product} as $\sbb$. 
Define the hash function $h_1, h_2: [p] \rightarrow [b]$ and $s_1, s2: [p] \rightarrow \{-1, 1\}$. Let $h$ be the hash function that satisfies $h(i,j) = h_1(i) + h_2(j) \mod b$, and $s$ be $s(i,j)  = s_1(i) s_2(j)$.
Let $\one_{i,j}^{q}$ be the indicator  of event $\{ h(i,j)=q \}$. 
Define the index set of the top $2k$ elements (with largest abstract value) of $\Ab\Bb^\top$  to be $\Psi_{2k}$. Denote the index set of the elements with absolute value greater than $\Delta$ as $\Psi_{\Delta}$. Let $\Psi = \Psi_{2k} \cap \Psi_{\Delta}$, and we are interested in finding all indices in $\Psi$. 

Our proof consists of the following four main steps. 

\textbf{Step I:} Bound the variance of a single decoded entry. 

\begin{align*}
    \sbb_q = \sum_{(i,j)\in [p]\times [p]}  \one_{i,j}^q s(i,j)(\Ab \Bb^\top )_{ij}
\end{align*}
For $(i^\star, j^\star)\in \Psi$, with $q^\star = h(i^\star, j^\star)$ we have:
\begin{align*}
    \sbb_{q^\star} = s(i^\star, j^\star) (\Ab\Bb^\top)_{i^\star j^\star} + \sum_{(i,j)\neq (i^\star, j^\star ), (i,j)\in [p]\times [p]} \one_{i,j}^{q^\star} s(i,j) (\Ab \Bb^\top )_{ij}
\end{align*}
Then, 
\begin{align*}
    |\sbb_{q^\star}| \ge & s(i^\star, j^\star) \sgn\left( \left( \Ab\Bb^\top\right)_{i^\star j^\star} \right) \sbb_{q^\star} \\
    =& \left\vert  (\Ab\Bb^\top)_{i^\star j^\star} \right\vert + s(i^\star, j^\star) \sgn\left( \left( \Ab\Bb^\top\right)_{i^\star j^\star} \right)  \sum_{(i,j)\neq (i^\star, j^\star ), (i,j)\in [p]\times [p]} \one_{i,j}^{q^\star} s(i,j) (\Ab \Bb^\top )_{ij}
\end{align*}
Let $\sbb_{q^\star}' = s(i^\star, j^\star) \sgn\left( \left( \Ab\Bb^\top\right)_{i^\star j^\star} \right)  \sum_{(i,j)\neq (i^\star, j^\star ), (i,j)\in [p]\times [p]} \one_{i,j}^{q^\star} s(i,j) (\Ab \Bb^\top )_{ij}$. We have:
\begin{align*}
    \PP \left(\mbox{ error in one bit }\right)  \le \PP \left( |\sbb_{q^\star}'| \ge \frac{\Delta}{2} \right) \le \frac{4\var(\sbb_{q^\star}')}{\Delta^2}
\end{align*}
Taking expectation over all possible partitions (based on $h$), we have:
\begin{align*}
    \EE_{h}\left[ \var(\sbb_{q^\star}') \right] = \frac{1}{b} \sum_{(i,j)\neq (i^\star, j^\star ), (i,j)\in [p]\times [p]} (\Ab\Bb^\top )_{ij}^2 \le \frac{\fnorm{\Ab\Bb^\top}^2}{b}. 
\end{align*}

\textbf{Step II:} Bound the failure probability of recovering a single large entry.

By Markov's inequality, we have
\begin{align*}
    \PP \left( \var(\sbb_{q^\star}') \ge \frac{c\fnorm{\Ab\Bb^\top}^2}{b} \right) \le \frac{1}{c}.
\end{align*}
Given the upper bound on $\var(\sbb_{q^\star}')$, which happens with probability at least $1-\frac{1}{c}$ due to the randomness from $h$, the only left randomness comes from $s(i,j)$. {Note that we use the same $h$ for every $t$ }. Then, 
\begin{align*}
    \PP \left( |\sbb_{q^\star}'| \ge \frac{\Delta}{2} \right) \le \frac{4c}{\Delta^2} \frac{\fnorm{\Ab\Bb^\top}^2}{b}
\end{align*}
The above inequality gives an error bound for each bit in the error-correcting code. Thus for a length $l$ code, the expected number of wrong bits is:
\begin{align*}
    \EE \left[ \mbox{ number of wrong bits }\right] = \sum_{r=0}^{2l-1} \PP \left( \mbox{ the $r^{\mathtt{th}}$ bit is wrong } \right) \le \frac{4lc}{\Delta^2} \frac{\fnorm{\Ab\Bb^\top}^2}{b}
\end{align*}
By using an expander code, we can tolerate a constant fraction of error $\delta$ which is independent of message length $\log p$, with a code length $l = \Ocal{(\log p)}$ \cite{sipser1994expander}. By Markov's inequality, and combining with the probability bound on $h$, 
\begin{align*}
    \PP \left( (i^\star, j^\star) \mbox{ not recovered }\right) \le \PP \left(\mbox{more than $\delta l$ wrong bits}\right) \le \frac{4c}{\Delta^2} \frac{\fnorm{\Ab\Bb^\top}^2}{ b\delta} + \frac{1}{c}
\end{align*}
Optimizing over the constant $c$ ({by setting $c=\frac{\Delta \sqrt{b\delta}}{2 \fnorm{\Ab\Bb^\top}}$}), we have 
\begin{align*}
    \PP \left( (i^\star, j^\star) \mbox{ not recovered }\right) \le \frac{4}{\Delta} \frac{\fnorm{\Ab\Bb^\top}}{ \sqrt{b\delta }} 
\end{align*}
By choosing $b\ge \frac{144\fnorm{\Ab\Bb^\top}^2}{\Delta^2 \delta}$, we have
\begin{align*}
    \PP \left( (i^\star, j^\star) \mbox{ not recovered }\right) \le \frac{1}{3}.
\end{align*}
For simplicity, we take $\delta = {1\over 3}$. Combining the assumption that $\fnorm{\Ab\Bb^\top}= \fnorm{\nabla f_{\Bcal_i}(\Thetab)}$. Taking  $\Delta \ge {\fnorm{\nabla f_{\Bcal}(\Thetab)}}/\sqrt{{2k}}$, which implies $b\ge{432\fnorm{\nabla f_{\Bcal_i}(\Thetab)}^2\over \Delta^2}$ will give a constant probability to successfully recover $(i^\star, j^\star)$.

\textbf{Step III:} Union bound over all large entries. 
Repeat the count sketch and sub-linear extraction for $d$ times and take the $(i,j)$ pair that are recovered more than $d/2$ times, we have that
\begin{align*}
    \PP \left( (i^\star, j^\star) \mbox{ not recovered for more than $d/2$ times}\right) \le \exp\left( -\frac{d}{48}\right).
\end{align*}
Since the events of recovering different $(i,j)\in \Psi$ are not independent (because of the dependency induced by $h$ functions), we  use  union bound over all the elements in $\Psi$. Thus we have 
\begin{align*}
    \PP \left( \Psi \mbox{ not recovered }\right) \le |\Psi| \exp\left(-\frac{d}{48}\right) \le 2k\exp \left(-\frac{d}{48}\right)
\end{align*}
By taking $d=48\log(2ck)$, we obtain the desired constant success rate $1-\frac{1}{c}$ for recovering $\Psi$. 

\textbf{Step IV:} For the overall time complexity of the Interaction Top Elements Extraction (\algsubroute), encoding the index will take $\Ocal(pl)$. Each compressed product step will take $\Ocal(mp+mb\log b)$ and it will be repeatedly calculated for $2l$ times, where $l$ is the length of the expander code. Given that expander code has a linear decoding complexity, thus the extraction step can be done with $\Ocal(bl)$. The above mentioned procedure will be repeated for $d$ times. Putting everything together, we have the time complexity for \algsubroute~ is
\begin{align*}
    \Ocal\left(\underbrace{\log (ck)}_{\mbox{repeat d times}}\left[\underbrace{p\log(p)+b\log(p)}_{\mbox{encode \& decode}} + \underbrace{\log(p)\left[mp+mb\log(b)\right]}_{\mbox{compressed product}}\right] \right)
\end{align*}
which achieves sub-quadratic time complexity. Ignoring the logarithm term, the time complexity is $\widetilde \Ocal(m(p+b))$, which naturally implies that the space complexity is $\widetilde \Ocal(m(p+b))$.
\end{proof}

\subsection{Proof of \Cref{lemma:Frobenius}}\label{prf:Frobenius}
\begin{proof}
	By RSM, we have
	\[
	\begin{split}
		\fnorm{ \nabla f_{\Bcal_t}(\Thetab) - \nabla f_{\Bcal_t}(\Thetab^\star) } \le L_{2k} \fnorm{\Thetab -  \Thetab^\star},\ \forall\Thetab\  s.t.  \left\vert \supp{(\Thetab)}\cup\supp{(\Theta^\star)} \right\vert \le 2k
	\end{split}
	\]
	By triangle inequality,
	\[
	\begin{split}
		\fnorm{\nabla f_{\Bcal_t}} \le L_{2k}\fnorm{\Thetab-\Thetab^\star}+\fnorm{\nabla f_{\Bcal_t}(\Thetab^\star)}
	\end{split}
	\]
	By the fact that $\fnorm{\Thetab}\le \sqrt{k}\omega$, the first term can be directly bounded by $L_{2k}\fnorm{\Thetab - \Thetab^\star}\le 2L_{2k}\sqrt{k}\omega$. For the last term we have $\fnorm{\nabla f_{\Bcal_t}(\Thetab^\star)}\le G$. Thus we have,
	\begin{align*}
		\fnorm{\nabla f_{\Bcal_t}(\Thetab)}\le 2L_{2k}\sqrt{k}\omega + G
	\end{align*}
\end{proof}

\subsection{Proof of \Cref{thm:iht}} \label{prf:iht}
\begin{proof}
~ \\
With stochastic gradient descent, we have $\Gb^t = \nabla f_{\Bcal t}(\Thetab^{t-1})$ as the gradient at step $t$. The per-round convergence can be separately analyzed for the two cases.
\newline
\textbf{\algsubroute~ succeeds: $\Lambda \subseteq \widetilde\Lambda$}. Before analyzing $\Thetab^t$, we first construct an intermediate parameter $\widetilde\Thetab^t$ as,
\begin{align*}
	\widetilde\Thetab^t = \Hcal_k\left({\Thetab^{t-1}-\eta\Pcal_{\Lambda_{2k}\cup\supp{(\Thetab^{t-1}})}}\left(\Gb_t\right)\right) = \Hcal_k\left(\Thetab^{t-1}-\eta\Gb_t\right)
\end{align*}
The second inequality directly comes from  \Cref{coro:HTproperty}. This is actually the best situation we can hope for. In this situation, the approximation projection in \algsubroute~ doesn't affect the update. We will start with the bound on $\fnorm{\widetilde\Thetab^t-\Thetab^\star}^2$. $\Thetab^t$ will then be compared with $\widetilde\Thetab^t$ to obtain the error bound. We will never refer to $\widetilde\Thetab^t$ in practice, but this construction makes the proof much clear. Consider the proxy 
\begin{align*}
    \Zb^t = \Thetab^{t-1} - \eta \Gb^t
\end{align*}
Let $\Omega = \supp(\Thetab^{t-1}) \cup \supp(\widetilde\Thetab^{t}) \cup \supp(\Thetab^\star)$, 
\begin{align}\label{eqt:guarantee-2}
    \fnorm{\widetilde\Thetab^t - \Thetab^\star}^2 = \fnorm{\Hcal_k\left(\Zb^t\right) - \Thetab^\star}^2 = \fnorm{\Hcal_k\left(\Pcal_{\Omega}\left(\Zb^t\right)\right) - \Thetab^\star}^2 \le \nu \fnorm{\Pcal_{\Omega}\left(\Zb^t\right) - \Thetab^\star}^2 
\end{align}
Notice that 
\begin{align}
    & \fnorm{\Pcal_{\Omega}\left(\Zb^t\right) - \Thetab^\star}^2 \\
    = & \fnorm{\Thetab^{t-1} - \Thetab^\star - \eta \Pcal_{\Omega}(\Gb^t)}^2 \nonumber \\
    \le & \fnorm{\Thetab^{t-1} - \Thetab^\star }^2 + \eta^2 \fnorm{\Pcal_{\Omega}(\Gb^t)}^2 - 2\eta \minnerprod{\Thetab^{t-1} - \Thetab^\star, \Gb^t} \label{eqt:guarantee-1}
\end{align}
Notice that $\EE[\Gb^t] = \nabla F(\Thetab^{t-1})$. \Cref{eqt:guarantee-1} includes three terms: (i) the first term is the contraction term which will be kept, (ii) the second term is controlled by first using \Cref{coro:proj-grad} then taking expectation, and (iii) the third term is controlled by first taking the expectation and then using the RSC property. 
Therefore, 
\begin{align*}
    & \EE_{\Bcal_t}\left[ \fnorm{\widetilde\Thetab^t - \Thetab^\star}^2 \right] \\
    \le & \EE_{\Bcal_t}\left[ \nu \fnorm{\Thetab^{t-1} - \Thetab^\star }^2 + \nu \eta^2 \fnorm{\Pcal_{\Omega}(\Gb^t)}^2 - 2\nu \eta \minnerprod{\Thetab^{t-1} - \Thetab^\star, \Gb^t} \right] \\
    \labelrel\le{rel:3} & \nu \fnorm{\Thetab^{t-1} - \Thetab^\star }^2   - 2\nu \eta \minnerprod{\Thetab^{t-1} - \Thetab^\star, \nabla F\left(\Thetab^{t-1}\right) }\\
    & + \nu\eta^2 \left[2L_{2k}^2\fnorm{\Thetab^{t-1} - \Thetab^{*}}^2 + 2\EE_{\Bcal_t}\left[\fnorm{\Pcal_\Omega\nabla f_{\Bcal_t}\left(\Thetab^{*}\right)}^2\right]\right]\\
    = & \nu \fnorm{\Thetab^{t-1} - \Thetab^\star }^2  + \nu\eta^2 \left[2L_{2k}^2\fnorm{\Thetab^{t-1} - \Thetab^{*}}^2 + 2\EE_{\Bcal_t}\left[\fnorm{\Pcal_\Omega\nabla f_{\Bcal_t}\left(\Thetab^{*}\right)}^2\right]\right]\\
    & - 2\nu \eta \minnerprod{\Thetab^{t-1} - \Thetab^\star, \nabla F\left(\Thetab^{t-1}\right)- \nabla F\left(\Thetab^\star\right) } + 2\nu \eta \minnerprod{\Thetab^{t-1} - \Thetab^\star, \nabla F\left(\Thetab^\star\right) } \\
    \labelrel\le{rel:4} & \nu \fnorm{\Thetab^{t-1} - \Thetab^\star }^2  + \nu\eta^2 \left[2L_{2k}^2\fnorm{\Thetab^{t-1} - \Thetab^{*}}^2 + 2\EE_{\Bcal_t}\left[\fnorm{\Pcal_\Omega\nabla f_{\Bcal_t}\left(\Thetab^{*}\right)}^2\right]\right]\\
    & - 2\nu \eta \alpha_{2k} \fnorm{\Thetab^{t-1}-\Thetab^\star}^2 + 2\nu \eta \fnorm{\Thetab^{t-1} - \Thetab^\star}\fnorm{\Pcal_{\Omega}\nabla F\left(\Thetab^\star\right) } \\
    = & \nu \left(1 - 2\eta \alpha_{2k} +  2\eta^2 L_{2k}^2 \right) \fnorm{\Thetab^{t-1} - \Thetab^\star}^2 \\
    & + 2\nu \eta \fnorm{\Thetab^{t-1} - \Thetab^\star}\fnorm{\Pcal_{\Omega}\left(\nabla F\left(\Thetab^\star\right)\right) }+ 2\nu \eta^2\EE_{\Bcal_t}\fnorm{\Pcal_{\Omega} \left(\nabla f_{\Bcal_t} \left(\Thetab^\star\right)\right)}^2
\end{align*}
where the first inequality is due to \Cref{eqt:guarantee-2} and \Cref{eqt:guarantee-1}. \eqref{rel:3} plugs in the result from \Cref{coro:proj-grad} and takes expectation over the gradient.  \eqref{rel:4} uses RSC property and Cauchy-Shwartz inequality. \newline
Suppose each coordinate of $\Thetab$ is bounded by $\omega$, we know that $\fnorm{\Thetab^{t-1}} \le \sqrt{k}\omega$ and $\fnorm{\Thetab^\star} \le \sqrt{k}\omega$, we further have
\begin{align*}
		 \EE_{\Bcal_t}\left[ \fnorm{\widetilde\Thetab^t - \Thetab^\star}^2 \right] \le &\nu \left(1 - 2\eta \alpha_{2k} +  2\eta^2 L_{2k}^2 \right) \fnorm{\Thetab^{t-1} - \Thetab^\star}^2 \\
		 &+ 4\nu \eta \sqrt{k}\omega\fnorm{\Pcal_{\Omega}\left(\nabla F\left(\Thetab^\star\right)\right) }+ 2\nu \eta^2\EE_{\Bcal_t}\fnorm{\Pcal_{\Omega} \left(\nabla f_{\Bcal_t} \left(\Thetab^\star\right)\right)}^2
\end{align*}
where the second line in the statistical error in SGD. With the definition of $\kappa_1, \sigma_{GD}^2$, 
\begin{align*}
	\EE_{\Bcal_t}\left[ \fnorm{\widetilde\Thetab^t - \Thetab^\star}^2 \right] \le \kappa_1 \fnorm{\Thetab^{t-1} - \Thetab^\star}^2 + \sigma_{GD}^2
\end{align*}
Now we turn to $\Thetab^t$ which is given by
\begin{align*}
	\Thetab^t = \Hcal_k\left({\Thetab^{t-1}-\eta\Pcal_{\widetilde\Lambda\cup\supp{(\Thetab^{t-1}})}}\left(\Gb_t\right)\right)
\end{align*}
It is very similar to $\widetilde\Thetab^t$, except that $\Lambda_{2k}$ is replaced by $\widetilde\Lambda$, which is the support we actually obtain.\newline
By definition, we have either $\Lambda_\Delta \subseteq \Lambda_{2k}$ or $\Lambda_{2k} \subseteq \Lambda_{\Delta}$. Recall that $\Lambda = \Lambda_{2k}\cap\Lambda_\Delta$ and in this case where \algsubroute~ recovers $\Lambda$, we have $\Lambda \subseteq \widetilde\Lambda$. Thus it is either $\Lambda_{2k}\subseteq\widetilde\Lambda$ or $\Lambda_{\Delta}\subseteq\widetilde\Lambda$.
\begin{enumerate}
	\item $\Lambda_{2k}\subseteq\widetilde\Lambda$. In this case, simply applying corollary \ref{coro:HTproperty} with $G = \Pcal_{\widetilde\Lambda\cup\supp{(\Thetab^{t-1}})}\left(\Gb_t\right)$, we have
	\begin{align*}
		\Thetab^t = \Hcal_k\left({\Thetab^{t-1}-\eta\Pcal_{\widetilde\Lambda\cup\supp{(\Thetab^{t-1}})}}\left(\Gb_t\right)\right) = \Hcal_k\left({\Thetab^{t-1}-\eta\Pcal_{\Lambda_{2k}\cup\supp{(\Thetab^{t-1}})}}\left(\Gb_t\right)\right)=\widetilde\Thetab^t
	\end{align*}
	Also, by $\Lambda_{2k}\subseteq\Lambda_{\Delta}$, we know that $k_{\Delta}=|\Lambda_{2k}\backslash\Lambda_{\Delta}|=0$, which indicates $\sigma_{\Delta|GD}^2=0$.
	\item $\Lambda_{\Delta}\subseteq\widetilde\Lambda$. Here we can apply \Cref{coro:deltaError} and have,
	\begin{align*}
		\norm{\Thetab^t - \widetilde\Thetab^t}_F\le 2\eta\Delta\sqrt{k_\Delta},~~~
		\norm{\Thetab^t - \widetilde\Thetab^t}_F^2\le 2\eta^2\Delta^2{k_\Delta}
	\end{align*}
	Thus, we can bound the error $\EE_{\Bcal_t}\left[\fnorm{\Thetab^t-\Thetab^\star}^2\right]$ as,
	\begin{align*}
		\EE_{\Bcal_t}\left[\fnorm{\Thetab^t-\Thetab^\star}^2\right] &= \EE_{\Bcal_t}\left[\fnorm{\widetilde\Thetab^t-\Thetab^\star}^2\right] + 2\EE_{\Bcal_t}\left[\minnerprod{\widetilde\Thetab^t-\Thetab^\star, \Thetab^t-\widetilde\Thetab^t} \right] + \EE_{\Bcal_t}\left[\norm{\Thetab^t - \widetilde\Thetab^t}_F^2\right]\\
		& \le \EE_{\Bcal_t}\left[\fnorm{\widetilde\Thetab^t-\Thetab^\star}^2\right] + 4\sqrt{k_\Delta}\eta \sqrt{k}\omega\Delta + 2k_\Delta\eta^2\Delta^2 \\
		& = \EE_{\Bcal_t}\left[\fnorm{\widetilde\Thetab^t-\Thetab^\star}^2\right] + \sigma^2_{\Delta|GD}
	\end{align*}
\end{enumerate}
Combining the two cases above, we have the desired convergence rate for $\Lambda \subseteq \widetilde\Lambda$.

\noindent \textbf{\algsubroute~ fails: $\Lambda \not\subset \widetilde\Lambda$.}  
This is the worst case when support recovery completely fails and we have no control over $\widetilde{\Lambda}$. The update in this case is
\begin{align*}
\Thetab^t = \Hcal_k \left( \widetilde{\Zb}^t \right), ~~{\Zb}^t = \Thetab^{t-1} - \eta\left[\Gb^t - \Pcal_{\widetilde{\Lambda}^C\backslash\supp(\Thetab^{t-1})}(\Gb^t)\right]
\end{align*}
Similar as the previous case, let $\Omega = \supp(\Thetab^{t-1}) \cup \supp(\Thetab^{t}) \cup \supp(\Thetab^\star)$, we have
\begin{align*}
\fnorm{\Thetab^t - \Thetab^\star}^2 \le \nu \fnorm{\Pcal_{\Omega}\left({\Zb}^t \right) - \Thetab^\star}^2
\end{align*}
and 
\begin{align*}
\fnorm{\Pcal_{\Omega}\left( {\Zb}^t \right) - \Thetab^\star}^2 
\le & \fnorm{ \Thetab^{t-1} - \Thetab^\star}^2 + \eta^2\fnorm{\Pcal_{\Omega}\left(\Gb^t- \Pcal_{\widetilde{\Lambda}^C\backslash\supp(\Thetab^{t-1})}(\Gb^t)\right)}^2 \\
& - 2\eta\minnerprod{ \Thetab^{t-1} - \Thetab^\star, \Pcal_{\Omega}(\Gb^t)}\\
& + 2\eta\minnerprod{ \Thetab^{t-1} - \Thetab^\star, \Pcal_{\Omega}\Pcal_{\widetilde{\Lambda}\backslash\supp(\Thetab^{t-1})}(\Gb^t)}\\
\le &  \fnorm{ \Thetab^{t-1} - \Thetab^\star}^2 + \eta^2\fnorm{\Pcal_{\Omega}\left(\Gb^t\right)}^2- 2\eta\minnerprod{ \Thetab^{t-1} - \Thetab^\star, \Pcal_{\Omega}(\Gb^t)}\\
& + 2\eta\minnerprod{ \Thetab^{t-1} - \Thetab^\star, \Pcal_{\Omega}\Pcal_{\widetilde{\Lambda}\backslash\supp(\Thetab^{t-1})}(\Gb^t)}
\end{align*}
The bound for the first three terms are same as the bound for \Cref{eqt:guarantee-1}. It left to bound the last term, 
\begin{align*}
	&\minnerprod{\Thetab^{t-1} - \Thetab^\star, \Pcal_{\Omega}\Pcal_{\widetilde{\Lambda}\backslash\supp(\Thetab^{t-1})}(\Gb^t)} \\
	\le &\fnorm{\Thetab^{t-1} - \Thetab^\star} \fnorm{\Pcal_{\Omega}\Pcal_{\widetilde{\Lambda}\backslash\supp(\Thetab^{t-1})}\left(\nabla f_{\Bcal_t}(\Thetab^{t-1})\right)} \\
	\le &\fnorm{\Thetab^{t-1} - \Thetab^\star} \fnorm{\Pcal_{\Omega}\left(\nabla f_{\Bcal_t}(\Thetab^{t-1})- \nabla f_{\Bcal_t}(\Thetab^{*})\right)} + \fnorm{\Thetab^{t-1} - \Thetab^\star} \fnorm{\Pcal_{\Omega}\left(\nabla f_{\Bcal_t}(\Thetab^{*})\right)}\\
	\le &\fnorm{\Thetab^{t-1} - \Thetab^\star} \fnorm{\nabla f_{\Bcal_t}(\Thetab^{t-1})- \nabla f_{\Bcal_t}(\Thetab^{*})} + \fnorm{\Thetab^{t-1} - \Thetab^\star} \fnorm{\Pcal_{\Omega}\left(\nabla f_{\Bcal_t}(\Thetab^{*})\right)}\\
	\le &L_{2k}\fnorm{\Thetab^{t-1} - \Thetab^\star}^2 + 2\sqrt{k}\omega\fnorm{\Pcal_{\Omega}\left(\nabla f_{\Bcal_t}(\Thetab^{*})\right)}
\end{align*}
Putting the bounds together, we have
\begin{align*}
\EE_{\Bcal_t}\left[ \fnorm{\Thetab^t - \Thetab^\star}^2 \right] \le & \nu\left(1-2\eta\alpha_{2k}+  2\eta^2 L_{2k}^2 + 2\eta L_{2k}\right)\fnorm{\Thetab^{t-1}-\Theta^\star}^2 \\
& + 4\nu \eta \sqrt{k}\omega\fnorm{\Pcal_{\Omega}\left(\nabla F\left(\Thetab^\star\right)\right) } + 2\nu\eta^2\EE_{\Bcal_t}\fnorm{\Pcal_{\Omega} \left(\nabla f_{\Bcal_t} \left(\Thetab^\star\right)\right)}^2 \\
& + 4\nu\eta\sqrt{k}\omega\EE_{\Bcal_{t}}\fnorm{\Pcal_{\Omega} \left(\nabla f_{\Bcal_t} \left(\Thetab^\star\right)\right)}
\end{align*}
Define $\sigma_{Fail|GD}^2=\max_{|\Omega|\le2k+K}\left[4\nu \eta\sqrt{k}\omega\EE_{\Bcal_t}\left[\fnorm{\Pcal_{\Omega}\left(\nabla f_{\Bcal_t}\left(\Thetab^\star\right)\right) }\right]\right]$
Then,
\begin{align*}
		 \EE_{\Bcal_t}\left[ \fnorm{\Thetab^t - \Thetab^\star}^2 \right] \le \nu \left(1 - 2\eta \alpha_{2k} +  2\eta^2 L_{2k}^2 + 2\eta L_{2k}\right) \fnorm{\Thetab^{t-1} - \Thetab^\star}^2 + \sigma_{GD}^2 + \sigma_{Fail|GD}^2
\end{align*}
\end{proof}

\subsection{Proof of \Cref{thm:main2}}\label{prf:main2}
\begin{proof}
	With the definition of $\sigma_1^2, \sigma_2^2, \kappa_1, \kappa_2$, the per-round convergence result of \Cref{thm:iht} can be rewritten as:
	\begin{enumerate}[leftmargin=*]
		\item Success Case:
		\begin{align*}
			\EE_{B_{t+1}}\left[\fnorm{\Thetab^{t+1}-\Thetab^\star}^2+{\sigma_1^2\over \kappa_1-1}\right]\le \kappa_1\EE_{B_{t}}\left[\fnorm{\Thetab^t - \Thetab^\star}^2+{\sigma_1^2\over \kappa_1-1}\right]
		\end{align*}
		\item Failure Case:
		\begin{align*}
			\EE_{B_{t+1}}\left[\fnorm{\Thetab^{t+1}-\Thetab^\star}^2+{\sigma_1^2\over \kappa_1-1}\right]\le \kappa_2\EE_{B_{t}}\left[\fnorm{\Thetab^t - \Thetab^\star}^2+{\sigma_1^2\over \kappa_1-1}\right] + \left(\kappa_2-1\right)\left({\sigma_2^2\over \kappa_2-1}-{\sigma_1^2\over\kappa_1-1}\right)
		\end{align*}
	\end{enumerate}
	For each iteration, the count sketch succeeds with probability $1-{1\over c}$. Denote the success indicator at iteration $t$ as $\phi_t$, and let $\Phi_t = \left\{\phi_0,\phi_1, ..., \phi_t\right\}$, we can combine those two cases and obtain,
	\begin{align*}
			\EE_{B_{t+1}, \Phi_{t+1}}\left[\fnorm{\Thetab^{t+1}-\Thetab^\star}^2+{\sigma_1^2\over \kappa_1-1}\right]\le &\left(\kappa_1 + {1\over c}\left(\kappa_1-\kappa_2\right)\right)\EE_{B_{t}, \Phi_{t}}\left[\fnorm{\Thetab^t - \Thetab^\star}^2+{\sigma_1^2\over \kappa_1-1}\right] \\
			&+ {1\over c}\left(\kappa_2-1\right)\left({\sigma_2^2\over \kappa_2-1}-{\sigma_1^2\over\kappa_1-1}\right)
		\end{align*}
	With a telescope sum, we have the desired error bound.
\end{proof}

\subsection{Proof of \Cref{thm:batchsize}}\label{prf:batchsize}
\begin{proof}
	We first vectorize the quadratic features. For an arbitrary data point ${\xb, y}$, we know that $\xb \in \RR^{p}$, where the first $p-1$ coordinates independently come from a zero mean, bounded distribution and the last coordinate is a constant 1. Denote the $i$-th coordinate of $\xb$ as $x_i$, we first vectorize the quadratic features, define
	\begin{align*}
		\vb = [x_1x_2, x_1x_3, ..., x_{p-1}x_p]^\top
	\end{align*}
	Here, for the quadratic terms, we only consider the interaction terms with no squared terms like $x_i^2$. Given that there will only be $p$ different squared terms, one can regress with the squared terms first, and the residual model will have no dependency on the squared terms. Replace $x_p$ with 1, we have
	\begin{align*}
		\vb = [x_1x_2, x_1x_3, ..., x_{p-2}x_{p-1}, x_1, ..., x_{p-1}]^\top
	\end{align*}
	Given that $x_1, ..., x_{p-1}$ are all i.i.d. and zero mean, after normalizing the variance of $x_i$, it is not hard to verify that
	\begin{align*}
		\EE[\vb\vb^\top] = \Ib
	\end{align*}
	Given that $\alpha, L$ are the smallest and largest eigenvalue of $\EE[\vb\vb^\top]$, asymptotically, we have the strong convexity and smoothness parameter $\alpha = L = 1$, which directly implies that the restricted version $\alpha_k = L_k = 1$. Thus the deterministic requirement can be easily satisfied with infinite sample. Now we turn to the minimum sample we need to have the desired $\alpha_k, L_k$.
	
	To show the $k$-restricted strong convexity and smoothness, we will first focus on an arbitrary $k$ sub-matrix of $\EE[\vb\vb^\top]$ and show the concentration. Then the desired claim will follow by applying a union bound over all $k$ sub-matrices.
	
	Denote a set of indices $\Scal \subseteq \{1, ..., |\vb|\}$, where $|\Scal| = k$. Define the corresponding sub-vector drawn from $\vb$ as $\zb = \vb_\Scal$. Define the restricted expected Hessian matrix as $\Hb = \EE[\zb\zb^\top]$, let the finite sample Hessian matrix as $\widehat\Hb = \frac{1}{m}\sum_{i=1}^m\zb_i\zb_i^\top$. Denote the difference as
	\begin{align*}
		\Db_i &= \frac{1}{m}\zb_i\zb_i^\top - \frac{1}{m}\EE[\zb\zb^\top] \\
		\Db &= \sum_{i=1}^m\Db_i = \frac{1}{m}\sum_{i=1}^m\zb_i\zb_i^\top - \EE[\zb\zb^\top] = \widehat\Hb - \Hb 
	\end{align*}
	
	Given that $\Hb = \Ib$, we can show the concentration of $\alpha_k, L_k$ as long as we can control $\Db$. Given that $\vb$ is bounded, we know that $\vb$ is sub-gaussian. Thus, bounding $\norm{\Db}$ is equivalent to showing concentration of the covariance matrix estimation of sub-gaussian random vectors. Using the Corollary 5.50 in \cite{vershynin2010introduction}, we have that with $m \ge C(t/\epsilon)^2k$, where $C$ depends only on the sub-gaussian norm $\norm{\vb}_{\psi_2}$. It's not hard to verify that $\norm{\vb}_{\psi_2} \le B$. Then we have that
	\begin{align*}
	    \PP(\norm{\Db} \ge \epsilon) \le 2\exp(-t^2k)
	\end{align*}
	
	Thus we obtain the bound for one particular $k$-sub-matrix. Taking an union bound over all $k$-sub-matrices, we have that
	\begin{align*}
		\PP(L_k \ge 1+\epsilon) &\le {|\vb|\choose k}\PP(\norm{\Db}\ge \epsilon)\\
		& \le \exp\left(k\log(2ep^2/k)-t^2k\right)
	\end{align*}
	
	By choosing $t^2 \ge \log(p)$, which implies that $m \gtrsim B k\log(p)/\epsilon^2$, we have $L_k\le 1+\epsilon$ with high probability. With a symmetric argument, we know that under same condition, we have $\alpha_k \ge 1-\epsilon$ with high probability.

\end{proof}

\subsection{Proof of \Cref{coro:sub-quadratic}}

\begin{proof}
    With \Cref{thm:main2} showing the linear convergence, we know that the per iteration complexity dominates the overall complexity of \algname. By \Cref{thm:top-recovery} and \Cref{lemma:Frobenius}, we show setting $b$ to $\Ocal(k)$ is sufficient for \algsubroute~to recover the support. \Cref{thm:batchsize} provides that the minimum batch size $m$ required for quadratic regression is $\Ocal({k\log p})$. Combining those results, we conclude that the complexity of \algname~is $\widetilde \Ocal (k(k+p))$. In the regime when $k$ is $\Ocal(p^\gamma)$ for $\gamma < 1$, \algname~achieves sub-quadratic complexity.
\end{proof}

\section{Proofs for \Cref{sec:extension}}

\subsection{Proof of \Cref{thm:SVRGPerRound}}\label{prf:svrgPerRound}
\begin{proof}
	Similar to the {\bf{\algsubroute~ succeeds}} case, define
	\begin{align*}
		\widetilde\Thetab^i_{j+1} = \Hcal_k\left({\Thetab^{i}_{j}-\eta\Pcal_{\Lambda_{2k}\cup\supp{(\Thetab^{i}_{j}})}}\left(\Gb^i_{j}\right)\right) = \Hcal_k\left(\Thetab^i_{j}-\eta\Gb^i_{j}\right)
	\end{align*}
	The second equality is given by \Cref{coro:HTproperty}. By applying \Cref{coro:deltaError}, we can also have
	\begin{align*}
		\norm{\Thetab^i_{j+1} - \widetilde\Thetab^i_{j+1}}_F \le 2\eta\Delta\sqrt{k_\Delta},~~~
		\norm{\Thetab^i_{j+1} - \widetilde\Thetab^i_{j+1}}_F^2 \le 2\eta^2\Delta^2{k_\Delta}
	\end{align*}
	which further implies
	\begin{align*}
		\fnorm{\Thetab^i_{j+1}-\Thetab^\star}^2 & =\fnorm{\Thetab^i_{j+1}-\widetilde\Thetab^i_{j+1}}^2+\fnorm{\widetilde\Thetab^i_{j+1}-\Thetab^\star}^2+2\minnerprod{\Thetab^i_{j+1}-\widetilde\Thetab^i_{j+1}, \widetilde\Thetab^i_{j+1}-\Thetab^\star}\\
		& \le \fnorm{\widetilde\Thetab^i_{j+1}-\Thetab^\star}^2 + 4\sqrt{k_\Delta}\eta \sqrt{k}\omega\Delta + 2k_\Delta\eta^2\Delta^2 \\
		& = \fnorm{\widetilde\Thetab^i_{j+1}-\Thetab^\star}^2 + \sigma^2_{\Delta|SVRG}
	\end{align*}
	The last equality defines $\sigma'^2_{\Delta|SVRG}$. To bound $\fnorm{\widetilde\Thetab^i_{j+1}-\Thetab^\star}^2$, the high level idea is similar to the proof of Theorem 10 in  \cite{shen2017tight}. We first define,
	\begin{align*}
    \Zb^i_{j+1} = \Thetab^{i}_j - \eta \Gb^i_j
	\end{align*}
	Let $\Omega = \supp(\Thetab^i_{j}) \cup \supp(\widetilde\Thetab^{i}_{j+1}) \cup \supp(\Thetab^\star)$, 
	\begin{align*}
    	\fnorm{\widetilde\Thetab^i_{j+1} - \Thetab^\star}^2 = \fnorm{\Hcal_k\left(\Zb^i_{j+1}\right) - \Thetab^\star}^2 = \fnorm{\Hcal_k\left(\Pcal_{\Omega}\left(\Zb^i_{j+1}\right)\right) - \Thetab^\star}^2 \le \nu \fnorm{\Pcal_{\Omega}\left(\Zb^i_{j+1}\right) - \Thetab^\star}^2 
	\end{align*}
	where the last inequality follows from \Cref{lemma:tightbound}. Thus we have
	\begin{align*}
    & \EE_{b_j}\left[ \fnorm{\widetilde\Thetab^i_{j+1} - \Thetab^\star}^2 \right] \\
    \le &  \EE_{b_j}\left[ \nu  \fnorm{\Pcal_{\Omega}\left(\Zb^i_{j+1}\right) - \Thetab^\star}^2 \right]  \\
    = & \EE_{b_j}\left[ \nu \fnorm{\Thetab^i_{j} - \Thetab^\star }^2 + \nu \eta^2 \fnorm{\Pcal_{\Omega}(\Gb^i_{j})}^2 - 2\nu \eta \minnerprod{\Thetab^i_{j} - \Thetab^\star, \Gb^i_{j}} \right]
    \end{align*}
    The second term can be bounded by using \Cref{coro:proj-grad} and we can take expectation directly on the third term, since $\EE_{b_j}[\Gb^i_{j}]=\nabla F(\Thetab^i_j)$. For brevity, denote $L=L_{|\Omega|}$.We then have,
    \begin{align*}
    & \EE_{b_j}\left[ \fnorm{\widetilde\Thetab^i_{j+1} - \Thetab^\star}^2 \right] \\
    \le & \nu \fnorm{\Thetab^i_{j} - \Thetab^\star }^2 + 4\nu\eta^2L_{2k}\left[F(\Thetab^i_j)-F(\Thetab^\star) + F(\Thetab^i_0) - F(\Thetab^\star)\right] - 2\nu \eta \minnerprod{\Thetab^i_{j} - \Thetab^\star, \nabla F\left(\Thetab^i_{j}\right) } \\
    & - 4\nu \eta^2 L_{2k} \minnerprod{\nabla F(\Thetab^\star), \Thetab^i_j+\Thetab^i_0-2\Thetab^\star} + 4\nu \eta^2\fnorm{\Pcal_{\Omega} \left(\nabla F \left(\Thetab^\star\right)\right)}^2\\
    \le & \nu(1-\eta\alpha_{2k}) \fnorm{\Thetab^i_{j} - \Thetab^\star }^2 - 2\nu\eta(1-2\eta L_{2k})\left[F(\Thetab^i_j)-F(\Thetab^\star)\right] + 4\nu\eta^2L_{2k}\left[F(\Thetab^i_0) - F(\Thetab^\star)\right] \\
    & + 4\nu \eta^2 L_{2k} \fnorm{\nabla F(\Thetab^\star)}\fnorm{\Thetab^i_j+\Thetab^i_0-2\Thetab^\star} + 4\nu \eta^2\fnorm{\Pcal_{\Omega} \left(\nabla F \left(\Thetab^\star\right)\right)}^2
\end{align*}
where the first inequality plugs in the result from \Cref{coro:proj-grad} and takes expectation of $\Gb^i_j$. The second inequality uses RSC property and Cauchy-Shwartz inequality. For brevity, define $\sigma' = \max_{|\Omega|=3k+K}\fnorm{\Pcal_{\Omega}F(\Thetab^*)}$, we have that
	\begin{align*}
		\EE_{b_j}\left[\fnorm{\widetilde\Thetab^i_{j+1}-\Thetab^\star}^2\right] \le & \nu(1-\eta\alpha_{2k})\fnorm{\Thetab^i_{j}-\Thetab^\star}^2-2\nu\eta(1-2\eta L_{2k})\left[F(\Thetab^i_{j})-F_(\Thetab^\star))\right]\\
		& + 4\nu\eta^2L_{2k}\left[F(\Thetab^i_0)-F(\Thetab^\star)\right] + 4\nu\eta \sigma'(4L_{2k}\sqrt{k}\omega+\sigma')
	\end{align*}
	Thus for $\EE_{b_j}\left[\fnorm{\Thetab^i_{j+1}-\Thetab^\star}^2\right]$, we have
	\begin{align*}
		\EE_{b_j}\left[\fnorm{\Thetab^i_{j+1}-\Thetab^\star}^2\right] \le & \nu(1-\eta\alpha_{2k})\fnorm{\Thetab^i_{j}-\Thetab^\star}^2-2\nu\eta(1-2\eta L_{2k})\left[F(\Thetab^i_{j})-F_(\Thetab^\star))\right]\\
		& + 4\nu\eta^2L_{2k}\left[F(\Thetab^i_0)-F(\Thetab^\star)\right] + 4\nu\eta \sigma'(4L_{2k}\sqrt{k}\omega+\sigma') + \sigma^2_{\Delta|SVRG}
	\end{align*}
	By a telescope sum, define $B_t=\left\{b_1, b_2,...,b_t\right\}$
	\begin{align*}
		\EE_{B_t}\left[\fnorm{\Thetab^{i}_t-\Thetab^\star}^2\right] \le & [\nu(1-\eta\alpha_{2k})-1]\sum_{j=0}^{t-1}\fnorm{\Thetab^i_{j+1}-\Thetab^\star}^2 + \fnorm{\Thetab^i_0-\Thetab^\star}\\
		& - 2\nu\eta(1-2\eta L_{2k})\sum_{j=0}^{t-1}\left[F(\Thetab^i_{j+1})-F(\Thetab^\star))\right]\\
		& + 4\nu\eta^2L_{2k}m\left[F(\Thetab^i_0)-F(\Thetab^\star)\right] + 4\nu\eta \sigma'(4L_{2k}\sqrt{k}\omega+\sigma')m + m\sigma^2_{\Delta|SVRG} \\
		= & [\nu(1-\eta\alpha_{2k})-1]m\EE_{B_t, j'}\fnorm{\Thetab^{i+1}_0-\Thetab^\star}^2 + \fnorm{\Thetab^i_0-\Thetab^\star}\\
		& - 2\nu\eta(1-2\eta L_{2k})\EE_{B_t, j'}\left[F(\Thetab^{i+1}_0)-F(\Thetab^\star))\right]\\
		& + 4\nu\eta^2L_{2k}m\left[F(\Thetab^i_0)-F(\Thetab^\star)\right] + 4\nu\eta \sigma'(4L_{2k}\sqrt{k}\omega+\sigma')m + m\sigma^2_{\Delta|SVRG}
	\end{align*}
	By using RSC, we have ${2\over\alpha_{2k}}\left[F(\Thetab^i_0)-F(\Thetab^\star)-\minnerprod{\nabla F(\Thetab^\star), \Thetab^i_0-\Thetab^\star} \right] \ge \fnorm{\Thetab^i_0-\Thetab^\star}^2$, thus
	\begin{align*}
		\EE_{B_t}\left[\fnorm{\Thetab^{i}_t-\Thetab^\star}^2\right] \le & [\nu(1-\eta\alpha_{2k})-1]m\EE_{B_t, j'}\fnorm{\Thetab^{i+1}_0-\Thetab^\star}^2 \\
		& - 2\nu\eta(1-2\eta L_{2k})\EE_{B_t, j'}\left[F(\Thetab^{i+1}_0)-F(\Thetab^\star))\right]\\
		& + \left({2\over\alpha_{2k}}+4\nu\eta^2L_{2k}m\right)\left[F(\Thetab^i_0)-F(\Thetab^\star)\right] \\
		& + 4\nu\eta \sigma'(4L_{2k}\sqrt{k}\omega+\sigma')m + {2\over \alpha_{2k}}\sqrt{k}\omega\sigma'+ m\sigma^2_{\Delta|SVRG}
	\end{align*}
	By assumption, we have $[\nu(1-\eta\alpha_{2k})-1] \le 0$. For simplicity, define
	\begin{align*}
		\sigma_{SVRG} = {4\nu\eta \sigma'(4L_{2k}\sqrt{k}\omega+\sigma')m + {2\over \alpha_{2k}}\sqrt{k}\omega\sigma'+ m\sigma^2_{\Delta|SVRG}\over  2 \nu\eta(1-2\eta L_{2k})m}
	\end{align*}
	Choosing $\eta<{1\over2L_{2k}}$, we have
	\begin{align*}
		\EE_{B_t}\left[F(\Thetab^{i+1}_0)-F(\Thetab^\star))\right] \le \left({1\over\alpha_{2k}\nu\eta(1-2\eta L_{2k})m}+{2\eta L_{2k}\over1-2\eta L_{2k}}\right)\left[F(\Thetab^i_0)-F(\Thetab^\star)\right] + \sigma_{SVRG}
	\end{align*}
	Typically $m$ is quite large, thus for the condition of convergence, we require that
	\begin{align*}
		{2\eta L_{2k}\over 1-2\eta L_{2k}} < 1\Rightarrow \eta < {1\over 4L_{2k}}
	\end{align*}
	Define,
	\begin{align*}
		\kappa_{SVRG} &= {1\over\alpha_{2k}\nu\eta(1-2\eta L_{2k})m}+{2\eta L_{2k}\over1-2\eta L_{2k}}
	\end{align*}
	We have the linear convergence given by
	\begin{align*}
		\EE_{B_t}\left[F(\Thetab^{i+1}_0)-F(\Thetab^\star))\right] \le \kappa_{SVRG}\left[F(\Thetab^{i}_0)-F(\Thetab^\star)\right]+\sigma_{SVRG}
	\end{align*}	
\end{proof}

\section{Proofs for \Cref{sec:analysis}}

\subsection{Proof of \Cref{lemma:tightbound}}\label{prf:tightbound}
\begin{proof}
We give the proof here for completeness. Also, the proof here is much concise than the original proof in \cite{shen2017tight} and a related result shown in \cite{li2016nonconvex}. 
\begin{align*}
	\fnorm{\Hcal_k(\Bb) - \Thetab}^2 = & \underbrace{\fnorm{\Pcal_{\supp(\Hcal_k(\Bb)) \backslash \supp(\Thetab)} \left( \Bb \right)}^2}_{\fnorm{\Bb_1}^2} + \underbrace{ \fnorm{\Pcal_{\supp(\Hcal_k(\Bb)) \cap \supp(\Thetab)} \left(\Bb - \Thetab\right)}^2 }_{\fnorm{\Bb_2 - \Thetab_2}^2}\\
    & +  \underbrace{ \fnorm{\Pcal_{  \supp^c(\Hcal_k(\Bb)) \cap \supp(\Thetab)  }\left(\Thetab\right) }^2 }_{\fnorm{\Thetab_3}^2}
\end{align*}
On the other hand,
\begin{align*}
\fnorm{\Bb - \Thetab}^2 = & \underbrace{ \fnorm{\Pcal_{\supp(\Hcal_k(\Bb)) \backslash \supp(\Thetab)} \left( \Bb \right)}^2}_{\fnorm{\Bb_1}^2} + \underbrace{ \fnorm{\Pcal_{\supp(\Hcal_k(\Bb)) \cap \supp(\Thetab)} \left(\Bb - \Thetab\right)}^2 }_{\fnorm{\Bb_2 - \Thetab_2}^2}\\
    & + \underbrace{\fnorm{\Pcal_{ \supp^c(\Hcal_k(\Bb)) \cap \supp(\Thetab) }\left(\Bb - \Thetab\right) }^2}_{\fnorm{\Bb_3 - \Thetab_3}^2}  + \underbrace{\fnorm{\Pcal_{ \supp^c(\Hcal_k(\Bb)) \backslash \supp(\Thetab)}\left(\Bb \right)}^2}_{\fnorm{\Bb_4}^2}
\end{align*}
\begin{align*}
&\max_{\Bb, \Thetab} \frac{\fnorm{\Hcal_k(\Bb) - \Thetab}^2}{\fnorm{\Bb - \Thetab }^2} \\
= &  \max_{\Bb, \Thetab} \frac{ \fnorm{\Bb_1}^2 + \fnorm{\Bb_2 - \Thetab_2}^2 + \fnorm{\Thetab_3}^2 }{ \fnorm{\Bb_1}^2 + \fnorm{\Bb_2 - \Thetab_2}^2 + \fnorm{\Bb_3 - \Thetab_3}^2 + \fnorm{\Bb_4}^2} \\
\le &  \max_{\Bb, \Thetab} \frac{ \fnorm{\Bb_1}^2 + \fnorm{\Bb_2 - \Thetab_2}^2 + \fnorm{\Thetab_3}^2 }{ \fnorm{\Bb_1}^2 + \fnorm{\Bb_2 - \Thetab_2}^2 + \fnorm{\Thetab_3}^2 + \fnorm{\Bb_3 }^2 - 2\minnerprod{\Bb_3, \Thetab_3}} \\
\le & \max\left\{ 1,   \max_{\Bb, \Thetab} \frac{  \left\vert \supp(\Bb_1) \right\vert {\Bb_{1,\min}}^2 + \fnorm{\Thetab_3}^2 }{  \left\vert \supp(\Bb_1) \right\vert {\Bb_{1,\min}}^2 + \fnorm{\Thetab_3}^2 + \fnorm{\Bb_3 }^2 - 2\minnerprod{\Bb_3, \Thetab_3}}   \right\} \\
\le & \max\left\{ 1,   \max_{\Bb, \Thetab  } \frac{  \left\vert \supp(\Bb_1) \right\vert {\Bb_{1,\min}}^2  + \fnorm{\Thetab_3}^2 }{   \left\vert \supp(\Bb_1) \right\vert {\Bb_{1,\min}}^2 + \fnorm{\Thetab_3}^2 +   \left\vert \supp(\Thetab_3) \right\vert  \Bb_{1,\min}^2 - 2\Bb_{1,\min} \onorm{\Thetab_{3} }   } \right\} \triangleq  \gamma  
\end{align*}
We determine $\gamma$ by observing
\begin{align*}
 & \frac{  \left\vert \supp(\Bb_1) \right\vert {\Bb_{1,\min}}^2  + \fnorm{\Thetab_3}^2 }{   \left\vert \supp(\Bb_1) \right\vert {\Bb_{1,\min}}^2 + \fnorm{\Thetab_3}^2 +   \left\vert \supp(\Thetab_3) \right\vert  \Bb_{1,\min}^2 - 2\Bb_{1,\min} \onorm{\Thetab_{3} }   } \le  \gamma,  ~~~~~~~~ \forall  \Bb, \Thetab \\
\Leftrightarrow &  (\gamma-1) \fnorm{\Thetab_3}^2 - 2\gamma \Bb_{1,\min}^2 + \left( \gamma \left\vert \supp(\Thetab_3)\right\vert + (\gamma -1) \left\vert \supp(\Bb_1) \right\vert \right)\Bb_{1,\min}^2 \ge 0, ~~~~~~~~ \forall  \Bb, \Thetab \\
\Leftrightarrow & 4\gamma^2 \Bb_{1,\min}^2 \le 4(\gamma - 1) \left[ \gamma + (\gamma-1) \frac{\supp(\Bb_1)}{\supp(\Thetab_3)} \right] \Bb_{1,\min}^2 \\
\Leftarrow &  \gamma^2  \le (\gamma - 1) \left[ \gamma + (\gamma-1) \frac{1}{\rho} \right] \\
\Leftarrow & \gamma = 1 + \frac{\rho + \sqrt{(4+\rho)\rho}}{2}, ~~~~~~~~ \mbox{ where } \rho=\frac{\min\{K, p^2-k\} }{k-K+\min\{K, p^2-k\} }
\end{align*}
\end{proof}

\subsection{Proof of \Cref{coro:co-cercivity}}\label{prf:co-cercivity}
\begin{proof}
\newcommand{\vara}{\Xib}
\newcommand{\varb}{\Xib'}
\newcommand{\varc}{\Thetab}
\newcommand{\vard}{\Thetab'}
\newcommand{\vare}{\vara}

Define the auxiliary function
\begin{align}\label{eqt:co-ercivity}
    g(\vara) := f(\vara) -  \minnerprod{ \nabla f(\varc), \vara } 
\end{align}
Notice that the gradient of $g(\cdot)$ satisfies:
\begin{align*}
    \fnorm{ \nabla g(\vara) - \nabla g(\varb) } = \fnorm{ \nabla f(\vara) - \nabla f(\varb) } \le L_{\zeronorm{\vara - \varb}} \fnorm{\vara - \varb}
\end{align*}
which implies
\begin{align*}
    g(\vara) - g(\varb) - \minnerprod{ \nabla g(\varb), \vara - \varb } \le \frac{L_r}{2} \fnorm{\vara- \varb}^2
\end{align*}
where $r=|\supp(\vara - \varb)|$. On the other hand, 
\begin{align*}
    g(\vare) - g(\varc) = f(\vare) - f(\varc) - \minnerprod{ \nabla f(\varc), \vare - \varc } \ge 0
\end{align*} 
as long as $f(\cdot)$ satisfies $|\supp(\vare) \cup \supp(\varc)|$-RC. 
Take $\vare = \vard - \frac{1}{L_{|\Omega|}} \Pcal_{\Omega} \nabla g(\vard) $, $\varb = \vard$, then, 
\begin{align*}
    g(\varc) \le & g(\vard - \frac{1}{L_{|\Omega|}} \Pcal_{\Omega} \nabla g(\vard)) \\
    \le & g(\vard) + \minnerprod{\nabla g(\vard), - \frac{1}{L_{|\Omega|}} \Pcal_{\Omega} \nabla g(\vard)} + \frac{1}{2 L_{|\Omega|}} \fnorm{\Pcal_{\Omega} \nabla g(\vard)}^2 \\
    = & g(\vard) - \frac{1}{2 L_{|\Omega|}} \fnorm{\Pcal_{\Omega} \nabla g(\vard)}^2
\end{align*}
Plug in the definition in \Cref{eqt:co-ercivity} gives the  result we want. 
\end{proof}

\subsection{Proof of \Cref{coro:proj-grad}}\label{prf:proj-grad}
\begin{proof} ~
\begin{enumerate}
	\item SGD:
	\begin{align*}
	& \EE_{\Bcal_t}\fnorm{\Pcal_{\Omega}\left(\Gb^t \right)}^2 = \EE_{\Bcal_t}\fnorm{ \Pcal_{\Omega}\left( \nabla f_{\Bcal_t} \left(\Thetab^{t-1}\right) \right) }^2 \\
    \le & 2 \EE_{\Bcal_t} \fnorm{\Pcal_{\Omega}\left( \nabla f_{\Bcal_t}\left( \Thetab^{t-1}\right) - \nabla f_{\Bcal_t}\left(\Thetab^\star\right) \right)}^2 + 2\fnorm{\Pcal_{\Omega} \left(\nabla f_{\Bcal_t} \left(\Thetab^\star\right)\right)}^2 \\
    \le & 2L_{2k}^2\fnorm{\Thetab^{t-1}-\Thetab^\star}^2+2\fnorm{\Pcal_\Omega\left(\nabla f_{\Bcal_t}(\Thetab^\star)\right)}^2
\end{align*}
The first inequality is by algebra, the second inequality holds by RSM.
	\item SVRG:
	\begin{align*}
		\fnorm{\Pcal_{\Omega}(\Gb^i_j)}^2 =& \fnorm{\Pcal_{\Omega}\left(\nabla f_{b_j} \left(\Thetab^{i}_j\right)-\nabla f_{b_j} \left(\Thetab^{i}_0\right)+\nabla F \left(\Thetab^{i}_0\right)\right)}^2 \\
		\le & 2 \fnorm{\Pcal_{\Omega}\left(\nabla f_{b_j} \left(\Thetab^{i}_j\right)-\nabla f_{b_j} \left(\Thetab^\star\right)\right)}^2 \\
		& + 2 \fnorm{\Pcal_{\Omega}\left(\nabla f_{b_j} \left(\Thetab^{i}_0\right)-\nabla f_{b_j} \left(\Thetab^\star\right)-\nabla F \left(\Thetab^{i}_0\right)\right)}^2
	\end{align*}
	Expand the later square, we have
	\begin{align*}
		\fnorm{\Pcal_{\Omega}(\Gb^i_j)}^2 \le & 2 \fnorm{\Pcal_{\Omega}\left(\nabla f_{b_j} \left(\Thetab^{i}_j\right)-\nabla f_{b_j} \left(\Thetab^\star\right)\right)}^2 + 2 \fnorm{\Pcal_{\Omega}\left(\nabla f_{b_j} \left(\Thetab^{i}_0\right)-\nabla f_{b_j} \left(\Thetab^\star\right)\right)}^2\\
		& + 2 \fnorm{\Pcal_\Omega\nabla F \left(\Thetab^{i}_0\right)}^2 - 4\minnerprod{\Pcal_{\Omega}\left(\nabla f_{b_j} \left(\Thetab^{i}_0\right)-\nabla f_{b_j} \left(\Thetab^\star\right)\right), \Pcal_\Omega\nabla F \left(\Thetab^{i}_0\right)}
	\end{align*}
	By applying \Cref{coro:co-cercivity} to bound the first two terms, we have
	\begin{align*}
		\fnorm{\Pcal_{\Omega}(\Gb^i_j)}^2 \le & 4L_{2k}\left[f_{b_j}(\Thetab^i_j)-f_{b_j}(\Thetab^\star)-\minnerprod{\nabla f_{b_j}(\Thetab^\star), \Thetab^i_{j}-\Thetab^\star}\right] \\
		& + 4L_{2k}\left[f_{b_j}(\Thetab^i_0)-f_{b_j}(\Thetab^\star)-\minnerprod{\nabla f_{b_j}(\Thetab^\star), \Thetab^i_{0}-\Thetab^\star}\right]\\
		& + 2 \fnorm{\Pcal_\Omega\nabla F \left(\Thetab^{i}_0\right)}^2 - 4\minnerprod{\Pcal_{\Omega}\left(\nabla f_{b_j} \left(\Thetab^{i}_0\right)-\nabla f_{b_j} \left(\Thetab^\star\right)\right), \Pcal_\Omega\nabla F \left(\Thetab^{i}_0\right)}
	\end{align*}
	Taking expectation over $b_j$, we have
	\begin{align*}
		\EE_{b_j}\fnorm{\Pcal_{\Omega}(\Gb^i_j)}^2 \le & 4L_{2k}\left[F(\Thetab^i_j)-F(\Thetab^\star)\right] + 4L_{2k}\left[F(\Thetab^i_0)-F(\Thetab^\star)\right]\\
		& -4L_{2k}\minnerprod{\nabla F(\Thetab^\star), \Thetab^i_{j}+\Thetab^j_0-2\Thetab^\star}\\
		& + 2\minnerprod{2\Pcal_{\Omega}\left(\nabla F \left(\Thetab^\star\right)\right) - \Pcal_{\Omega}\left(\nabla F \left(\Thetab^{i}_0\right)\right), \Pcal_\Omega\nabla F \left(\Thetab^{i}_0\right)}\\
		= & 4L_{2k}\left[F(\Thetab^i_j)-F(\Thetab^\star)\right] + 4L_{2k}\left[F(\Thetab^i_0)-F(\Thetab^\star)\right]\\
		& -4L_{2k}\minnerprod{\nabla F(\Thetab^\star), \Thetab^i_{j}+\Thetab^j_0-2\Thetab^\star}\\
		& + \fnorm{2\Pcal_\Omega\left(\nabla F(\Thetab^\star)\right)}^2-\fnorm{2\Pcal_\Omega(\nabla F(\Thetab^\star)-\nabla F(\Thetab^i_0))}^2\\
		& - \fnorm{\Pcal_{\Omega}\left(\nabla F(\Thetab^i_0)\right)}^2 \\
		\le & 4L_{2k}\left[F(\Thetab^i_j)-F(\Thetab^\star)\right] + 4L_{2k}\left[F(\Thetab^i_0)-F(\Thetab^\star)\right]\\
		& -4L_{2k}\minnerprod{\nabla F(\Thetab^\star), \Thetab^i_{j}+\Thetab^j_0-2\Thetab^\star} + 4\fnorm{\Pcal_\Omega\left(\nabla F(\Thetab^\star)\right)}^2
	\end{align*}
\end{enumerate} 
\end{proof}

\subsection{Proof of \Cref{coro:HTproperty}}\label{prf:HTproperty}
\begin{proof}
	Denote $\Thetab^+=\Hcal_k \left( \Thetab - \eta \Gb \right)$. Define $\Lambda_{new}$ to be the indices set of $k$-largest elements in $G$ that doesn't belong to $\supp(\Thetab)$. It can be easily verified that
	\begin{align*}
		\Hcal_k \left( \Thetab - \eta \Gb \right) = \Hcal_k \left( \Thetab - \eta\Pcal_{\supp(\Thetab)\cup\Lambda_{new}}\left(\Gb\right) \right)
	\end{align*}
	Given that $|\supp(\Thetab)|\le k$, by pigeonhole principle, we have $\Lambda_{new}\subseteq\Lambda_{2k}$, thus
	\begin{align*}
		\Hcal_k \left( \Thetab - \eta\Pcal_{\supp(\Thetab)\cup\Lambda_{new}}\left(\Gb\right) \right) = \Hcal_k\left(\Thetab - \eta \Pcal_{\supp(\Thetab) \cup \Lambda_{2k}} (\Gb) \right)
	\end{align*}
\end{proof}

\subsection{Proof of \Cref{coro:deltaError}}\label{prf:deltaError}
\begin{proof}
	Define
	\begin{align*}
		\Gamma_0 = \supp(\Thetab^+)\cap\supp(\widetilde\Thetab^+), \Gamma_1 = \supp(\Thetab^+)\backslash\supp(\widetilde\Thetab^+), \Gamma_2=\supp(\widetilde\Thetab^+)\backslash\supp(\Thetab^+)
	\end{align*}
	We have that
	\begin{align*}
		\Thetab^+-\widetilde\Thetab^+ = \Pcal_{\Gamma_1}(\Thetab^+) - \Pcal_{\Gamma_2}(\widetilde\Thetab^+)
	\end{align*}
	By definition of $\Thetab^+, \widetilde\Thetab^+$, it is easy to verify that
	\begin{align*}
		\norm{\Pcal_{\Gamma_1}(\Thetab^+)}_\infty \le \norm{\Pcal_{\Gamma_2}(\widetilde\Thetab^+)}_\infty
	\end{align*}
	Also, since the elements that is greater than $\eta\Delta$ can only come from $\supp(\Thetab)\cup \Lambda_\Delta$, and given that $\Lambda_{\Delta}\subseteq\widetilde\Lambda$, we know that
	\begin{align*}
		i \in \supp(\Thetab^+), \forall i\ s.t. \widetilde\Thetab^+_i\ge \eta\Delta \Rightarrow i \in \Gamma_0, \forall i\ s.t. \widetilde\Thetab^+_i\ge \eta\Delta
	\end{align*}
	Thus we have
	\begin{align*}
		\norm{\Pcal_{\Gamma_2}(\widetilde\Thetab^+)}_\infty \le \eta\Delta
	\end{align*}
	By \Cref{coro:HTproperty}, we know that
	\begin{align*}
		\widetilde\Thetab^+=\Hcal_{k}\left(\Thetab-\eta\Pcal_{\Lambda_{2k}\cup\supp({\Thetab})}(\Gb_t)\right)
	\end{align*}
	Thus we have $|\Gamma_2|\le k_\Delta$ and given that $|\supp(\Thetab^+)| = |\widetilde\supp(\Thetab^+)|$, we have $|\Gamma_1|=|\Gamma_2|$. Thus,
	\begin{align*}
		\fnorm{\Thetab^+-\widetilde\Thetab^+} = \fnorm{\Pcal_{\Gamma_1}(\Thetab^+) - \Pcal_{\Gamma_2}(\widetilde\Thetab^+)} \le \eta\Delta \sqrt{2k_\Delta}
	\end{align*}
\end{proof}

\end{document}